%% file: 0-main.tex
\def\year{2022}\relax
\documentclass[letterpaper]{article} 
\usepackage{aaai22}  
\usepackage{times}  
\usepackage{helvet}  
\usepackage{courier}  
\usepackage[hyphens]{url}  
\usepackage{graphicx} 
\usepackage{natbib}  
\usepackage{caption} 
\DeclareCaptionStyle{ruled}{labelfont=normalfont,labelsep=colon,strut=off} 
\usepackage{amsfonts}
\usepackage{amsmath}
\usepackage[ruled,vlined,linesnumbered]{algorithm2e}
\usepackage{amsthm}

\usepackage{gensymb}
\usepackage{multirow}
\usepackage{makecell}
\usepackage[normalem]{ulem} 
\usepackage{graphicx}

\newtheorem{lemma}{Lemma}
\newtheorem{theorem}{Theorem}
\newtheorem{assumption}{Assumption}
\newtheorem{corollary}{Corollary}

\usepackage{xcolor}

\title{FedSoft: Soft Clustered Federated Learning with Proximal Local Updating}
\author {
    Yichen Ruan, Carlee Joe-Wong\\
}
\affiliations {
    Carnegie Mellon University\\
    yichenr@andrew.cmu.edu, cjoewong@andrew.cmu.edu
}

\begin{document}

\maketitle

\begin{abstract}
Traditionally, clustered federated learning groups clients with the same data distribution into a cluster, so that every client is uniquely associated with one data distribution and helps train a model for this distribution. We relax this hard association assumption to soft clustered federated learning, which allows every local dataset to follow a mixture of multiple source distributions. We propose \textbf{FedSoft}, which trains both locally personalized models and high-quality cluster models in this setting. \textbf{FedSoft} limits client workload by using proximal updates to require the completion of only one optimization task from a subset of clients in every communication round. We show, analytically and empirically, that \textbf{FedSoft} effectively exploits similarities between the source distributions to learn personalized and cluster models that perform well.
\end{abstract}

\input{1-introduction}
\input{2-review}
\input{3-algorithm}
\input{4-analysis}
\input{5-experiments}
\input{6-conclusion}

\newpage
\section*{Acknowledgments}
This research was partially supported by NSF CNS-1909306 and CNS-2106891.

\bibliography{0-main.bib}
\newpage
\input{7-appendix}

\end{document}

%% file: 1-introduction.tex
\section{Introduction}\label{sec:intro}
Federated learning (FL) is an innovative privacy-preserving machine learning paradigm that distributes collaborative model training across participating user devices without users' sharing their raw training samples. In the widely used federated learning algorithm \textbf{FedAvg} \cite{fedavg}, clients jointly train a shared machine learning model by iteratively running local updates and synchronizing their intermediate local models with a central server. In spite of its success in applications such as next word prediction \cite{next_word} and learning on electronic health records \cite{medical_fl}, FL is known to suffer from slow model training when clients' local data distributions are heterogeneous, or non-IID (non- independently and identically distributed).
In response to this challenge, some recent works propose to bypass data heterogeneity by performing \emph{local model personalization}. Instead of pursuing one universally applicable model shared by all clients, these algorithms' training objective is to create one model for each client that fits its local data. Personalization methods include local fine tuning \cite{fine_tune}, model interpolation \cite{three_approaches}, and multi-task learning \cite{fmtl}. In this paper, we focus on an alternative approach: \emph{clustered federated learning}, which we generalize to train both cluster and personalized models on realistic distributions of client data. 

Clustered FL relaxes the assumption of FL that each client has an unique data distribution; instead, it allows different clients to share one data distribution, with fewer source data distributions than clients. The objective of clustered FL is to train one model for every distribution. In traditional clustered FL, a client can only be associated with one data distribution. We thus call this method \emph{hard clustered federated learning}. Under the hard association assumption, the non-IID problem can be easily resolved: simply group clients with the same data distribution into one cluster, then conduct conventional FL on each cluster, within which the data distribution is now IID among clients. 
%
Unlike other personalization methods, clustered FL thus produces centrally available models that can be selectively migrated to new users that are unwilling, or unable, to engage in the subsequent local adaptation process (e.g. fine tuning) due to privacy concerns or resource limitations. This convenience in model adoption is particularly valuable for the current training-testing-deployment lifecycle of FL where deployment, rather than the training itself, is the end goal \cite{advances}.

However, hard clustered FL faces two fundamental problems in practice. First, \emph{multiple clients may be unlikely to possess identical data distributions.} 
In fact, the real-world user data is more likely to follow a \emph{mixture} of multiple distributions \cite{fmtl_mixture}. 
E.g., if each client is a mobile phone and we wish to model its user's content preferences, we might expect the clients to be clustered into adults and children. However, adult users may occasionally view children's content, and devices owned by teenagers (or shared by parents and children) may possess large fractions of data from both distributions. Similarly, content can be naturally grouped by users' interests (e.g., genres of movies), each of which may have a distinct distribution. 
Data from users with multiple interests then reflects a mixture of these distributions. Since the mixture ratios can vary for different clients, they may have different overall distributions even though the source distributions are identical. Clustering algorithms like the Gaussian mixture model \cite{gaussian_mixture} use a similar rationale. Clients may then require models personalized to their distributions to make accurate predictions on their data, in addition to the cluster models used for new users. 

Hard clustered FL's second challenge is that \emph{it cannot effectively exploit similarities between different clusters}. Though FL clients may have non-IID data distributions, two different distributions may still exhibit some similarity, as commonly assumed in personalization works \cite{fmtl}. For example, young people may have more online slang terms in their chatting data, but all users (generally) follow the same basic grammar rules. 
Thus, the knowledge distilled through the training on one distribution could be transferred to accelerate the training of others. However, in most hard clustered FL algorithms, different cluster models are trained independently
, making it difficult to leverage the potential structural likeness among distributions. Note that unlike in other personalization methods where the discussion of similarity is restricted to similarities between individual clients, here we focus on the broader similarities between source cluster distributions. Thus, we can gain better insight into the general data relationship rather than just the relationships between participating clients.

To overcome clustered FL's first challenge of the hard association assumption, in this paper, we utilize \emph{soft clustered federated learning}. In soft clustered FL, we suppose that the data of each client follows a mixture of multiple distributions. 
However, training cluster models using clients with mixed data raises two new challenges. First, \emph{the workload of clients can explode}. When all the data of a client comes from the same distribution, as in hard clustered FL, it ideally only needs to contribute towards one training task: training that distribution's model. 
However, in soft clustered FL, a client has multiple data sources. 
A natural extension of hard clustered FL is for the client to help train all cluster models whose distributions are included in its mixture \cite{fmtl_mixture}. However, the workload of participating clients then grows linearly with the number of clusters, which can be large (though typically much smaller than the number of clients) for some applications. This multiplying of client workload can make soft clustered FL infeasible, considering the resource restrictions on typical FL user devices and the long convergence time for many FL models \cite{fedavg}. Second, \emph{the training of cluster models and the local personalization are distinct}. In hard clustered FL, client models are the same as cluster models since a client is uniquely bound to one cluster. In soft clustered FL, local distributions differ from individual cluster distributions, and thus training cluster models does not directly help the local personalization. Complicating things further, these local distributions and their exact relationships to the cluster models are unknown a priori. Combining the training of cluster and personalized models is then challenging. 

To solve these two challenges, and handle the second disadvantage of hard clustered FL discussed above, we utilize the proximal local updating trick, which is originally developed in \textbf{FedProx} \cite{fedprox} to grant clients the use of different local solvers in FL. During the course of proximal local updating, instead of working on fitting the local model to the local dataset, each client optimizes a proximal local objective function that both carries local information and encodes knowledge from all cluster models. We name this proposed algorithm \textbf{FedSoft}. 

In \textbf{FedSoft}, since the fingerprints of all clusters are integrated into one optimization objective, clients only need to solve \emph{one single optimization problem}, for which the workload is almost the same as in conventional FL. In addition, by combining local data with cluster models in the local objective, clients can \emph{perform local personalization on the fly}. Eventually, the server obtains collaboratively trained cluster models that can be readily applied to new users, and each participating client gets one personalized model as a byproduct. Proximal local updating allows a cluster to \emph{utilize the knowledge of similar distributions}, overcoming the second disadvantage of the hard clustered FL. Intuitively, with all clusters present in the proximal objective, a client can take as reference training targets any cluster models whose distributions take up non-trivial fractions of its data. These component distributions, co-existing in the same dataset, are similar by nature. Thus, a personalized local model integrating all its component distributions can in turn be utilized by the component clusters to exploit their similarities. 



Our \emph{contributions} are: We design the \textbf{FedSoft} algorithm for efficient soft clustered FL. We establish a convergence guarantee that relates the algorithm's performance to the divergence of different distributions, and validate the effectiveness of the learned cluster and personalized models in experiments under various mixture patterns. Our results show the training of cluster models converges linearly to a remaining error determined by the cluster heterogeneity, and that \textbf{FedSoft} can outperform existing FL implementations in both global cluster models for future users and personalized local models for participating clients. 

%% file: 2-review.tex
\section{Related Works} \label{sec:review}
The training objective of hard clustered FL is to simultaneously identify the cluster partitions and train a model for each cluster. Existing works generally adopt an Expectation-Maximization (EM) like algorithm, which iteratively alternates between the cluster identification and model training. Based on how the partition structure is discovered, these algorithms can be classified into four types: 

The first type leverages the distance between model parameters, e.g., \citet{multi-center} propose to determine client association based on the distances between client models and server models. Similarly, \citet{flhc} suggest to apply a distance-based hierarchical clustering algorithm directly on client models. The second type determines the partition structure based on the gradient information, e.g., the \textbf{CFL} \cite{cfl} algorithm splits clients into bi-partitions based on the cosine similarity of the client gradients, and then checks whether a partition is congruent (i.e., contains IID data) by examining the norm of gradients on its clients. Likewise, the \textbf{FedGroup} \cite{fedgroup} algorithm quantifies the similarity among client gradients with the so-called Euclidean distance of decomposed cosine similarity metric, which decomposes the gradient into multiple directions using singular value decomposition. The third type utilizes the training loss, e.g., in \textbf{HyperCluster} \cite{three_approaches}, each client is greedily assigned to the cluster whose model yields the lowest loss on its local data. A generalization guarantee is provided for this algorithm. \citet{hard_cluster_ucb} propose a similar algorithm named \textbf{IFCA}, for which a convergence bound is established under the assumption of good initialization and all clients having the same amount of data. The fourth type uses exogenous information about the data, e.g., \citet{patient_cluster} and \citet{covid} group patients into clusters respectively based on their electronic medical records and imaging modality. This information usually entails direct access to the user data and thus cannot be applied in the general case.

Recently, \citet{fmtl_mixture} propose a multi-task learning framework similar to soft clustered FL that allows client data to follow a mixture of distributions. Their proposed ~~~\textbf{FedEM} algorithm adopts an EM algorithm and estimates the mixture coefficients based on the training loss. However, ~\textbf{FedEM} 
requires every client to perform a local update for each cluster in each round, which entails significantly more training time than conventional \textbf{FedAvg}. Their analysis moreover assumes a special form of the loss function with all distributions having the same marginal distribution, which is unrealistic. In contrast, \textbf{FedSoft} requires only a subset of clients to return gradients for only one optimization task in each round. Moreover, we show its convergence for generic data distributions and loss functions.

The proximal local updating procedure that we adopt incorporates a regularization term in the local objective, which is also used for model personalization outside clustered settings. Typical algorithms include \textbf{FedAMP} \cite{fedamp}, which adds an attention-inducing function to the local objective, and \textbf{pFedMe} \cite{pfedme}, which formulates the regularization as Moreau envelopes.

%% file: 3-algorithm.tex
\section{Formulation and Algorithm} \label{sec:algorithm}
\textbf{Mixture of distributions}. Assume that each data point at each client is drawn from 
\emph{one of} the $S$ distinct data distributions $\mathcal{P}_1,\cdots,\mathcal{P}_S$. Similar to general clustering problems, we take $S$ as a hyperparameter determined a priori. Data points from all clients that follow the same distribution form a \emph{cluster}. In soft clustered FL, a client may possess data from multiple clusters. Given a loss function $l(w;x,y)$, the (real) \emph{cluster risk} $F_s(w)$ is the expected loss for data following $\mathcal{P}_s$:
 \begin{equation} \label{eq:Fs}
    F_s(w) \overset{\Delta}{=} \mathbb{E}_{(x,y) \sim \mathcal{P}_s}[l(w;x,y)]
 \end{equation}
 
 We then wish to find $S$ cluster models $c_1^*$
 $\cdots c_S^*$ such that all cluster objectives are minimized simultaneously. These cluster models will be co-trained by all clients through coordination at the central server:
 \begin{equation} \label{eq:cs*}
     c_s^* = \text{argmin}_w F_s(w), s = 1,\cdots,S
 \end{equation}

Suppose a client $k \in [N]$ with local dataset $\mathcal{D}_k$ has $|\mathcal{D}_k| = n_k$ data points, among which $n_{ks}$ data points are sampled from distribution $\mathcal{P}_s$. The real risk of a client can thus be written as an average of the cluster risks:
\begin{equation} \label{eq:fk}
\begin{split}
    &f_k(w_k) \overset{\Delta}{=} \frac{1}{n_k}\mathbb{E}\left[\sum_{s=1}^S \sum_{(x_{k}^i,y_{k}^i) \sim \mathcal{P}_s} l(w_k;x_k^i,y_k^i)\right] \\
    =&  \frac{1}{n_k} \sum\nolimits_{s=1}^S n_{ks} F_s(w_k) = \sum\nolimits_{s=1}^S u_{ks}F_s(w_k)
\end{split}
\end{equation}

Here we define $u_{ks} \overset{\Delta}{=} n_{ks}/n_k \in [0,1]$ as the \emph{importance weight} of cluster $s$ to client $k$. In general, $u_{ks}$'s are unknown in advance and the learning algorithm attempts to estimate their values during the learning iterations. It is worth noting that while we directly work on real risks, our formulation and analysis can be easily extended to empirical risks by introducing local-global divergences as in \cite{fedprox}.

\textbf{Proximal local updating}. Since $f_k$ is a mixture of cluster risks, minimizing (\ref{eq:fk}) alone does not help solve (\ref{eq:cs*}). Thus, we propose each client instead optimize a proximal form of (\ref{eq:fk}):
\begin{equation} \label{eq:hkt}
    h_k(w_k;c^t,u^t) \overset{\Delta}{=} f_k(w_k) + \frac{\lambda}{2} \sum\nolimits_{s=1}^S u_{ks}^t \|w_k - c_s^t \|^2
\end{equation}
 
Here $\lambda$ is a hyperparameter and $u_{ks}^t$ denotes the estimation of $u_{ks}$ at time $t$. In the local updating step, every client searches for the optimal local model $w_k^*$ that minimizes $h_k$ given the current global estimation of cluster models $\{c_s^t\}$. As in \cite{fedprox}, clients may use any local solver to optimize $h_k$. This design of the proximal objective entails cluster models $\{c_s^t\}$ be shared among all clients through the server, as is usual in clustered FL \cite{hard_cluster_ucb}. We thus alternatively call $\{c_s^t\}$ the \emph{centers}.

The regularization term $\frac{\lambda}{2} \sum_{s=1}^S u_{ks}^t \|w_k - c_s^t \|^2$ in the proximal objective serves as a reference point for the local model training. It allows clients to work on their own specific dataset while taking advantage of and being guided by the globally shared knowledge of the centers. The regularization is weighted by the importance weights $u_{ks}$, so that a client will pay more attention to distributions that have higher shares in its data. To see why compounded regularization helps identify individual centers, assume we have a perfect guess of $u_{ks}^t \equiv u_{ks}$. The minimization of (\ref{eq:hkt}) can then be decoupled as a series of sub- optimization problems $h_k(w_k;c^t) = \sum_{s=1}^S u_{ks} \left(F_s(w_k) + \frac{\lambda}{2} \|w_k - c_s^t \|^2\right)$. Thus, after $h_k$ is minimized, the sub-problems corresponding to large $u_{ks}$ will also be approximately solved. We can hence utilize the output local model ${w_k^{t}}^*$ to update these centers with large $u_{ks}$. Moreover, ${w_k^{t}}^*$ trained in this manner forges all its component distributions $\left\{\mathcal{D}_s | u_{ks}^t \neq 0\right\}$, which may share some common knowledge. Thus, the training of these clusters are bonded through the training of their common clients, exploiting similarities between the clusters.

The output model ${w_k^{t}}^*$ is itself a well personalized model that leverages both local client knowledge and the global cluster information. \citet{fmtl_mixture} show that under certain conditions, the optimal client model for soft clustered FL is a mixture of the optimal cluster models. The same implication can also be captured by our proximal updating formulation. When $\sum_s u_{ks}^t = 1$, the gradient $\nabla h_k$ is
\begin{equation}
    \nabla_{w_k} h_k = \nabla f_k(w_k) + \lambda \left( w_k - \sum\nolimits_{s=1}^S u_{ks}^t c_s^t \right)
\end{equation}
which implies that $w_k^*$ should be centered on $\sum_s u_{ks} c_s^*$. As a result, through the optimization of all $h_k$, not only will the server obtain the trained cluster models, but also each client will obtain a sufficiently personalized local model.

\begin{algorithm}[h]
\SetAlgoLined
\textbf{Input}:
Global epoch $T$, importance weights estimation interval $\tau$, number of clients $N$, client selection size $K$, counter smoother $\sigma$\;
 \For{$t = 0,\cdots, T-1$}{
    \eIf{$t~~mod~~\tau = 0$}{
        Server sends centers $\{c_s^t\}$ to all clients \;
        \For{each client $k$}{
            \For{each data point $(x_k^i, y_k^i)$}{
                $j = \text{argmin}_s l(c_s^t;x_k^i,y_k^i)$ \;
                $n_{kj}^t = n_{kj}^t +  1$ \;
            }
            Send $u_{ks}^t = \max \{\frac{n_{ks}^t}{n_k}, \sigma \}$ to server \;
        }
    }{
        Set $u_{ks}^t = u_{ks}^{t-1}$ \;
    }
    Server computes $v_{sk}^t$ as in (\ref{eq:vskt}) \;
    
    Server selects $S$ sets of clients $Sel_s^t \subset [N]$ at random for each cluster, where $|Sel_s^t| = K$, and each client gets selected with probability $v_{sk}^{t}$\;
    
    Selected clients download $\{c_s^t\}$, then compute and report $w_k^{t+1} = \text{argmin}_{w_k} h_k(w_k;c^t,u^t)$ \;
    
    Server aggregates $c_s^{t+1} = \frac{1}{K}\sum_{k \in Sel_s^t} w_k^{t+1}$ \;
 }
 \caption{FedSoft} \label{algo:proposed}
\end{algorithm}

\textbf{Algorithm design}. We formally present \textbf{FedSoft} in Algorithm \ref{algo:proposed}. The first step of the algorithm is to estimate the importance weights $\{u_{ks}^t\}$ for each client $k$ (lines 3-14). The algorithm obtains them by finding the center that yields the  smallest loss value for every data point belonging to that client, and counting the number of points $\{n_{ks}^t\}$ matched to every cluster $s$. If a client $k$ has no samples matched to $s$ ($n_{ks}^t = 0$), the algorithm sets $u_{ks}^t = \sigma$, where $0 < \sigma \ll 1$ is a pre-defined smoothing parameter.

Once the server receives the importance weights $\{u_{ks}^t\}$, it computes the \emph{aggregation weights} $v_{sk}^t$ as follows (line 15):
\begin{equation} \label{eq:vskt}
    v_{sk}^t = \frac{u_{ks}^tn_k}{\sum_{k' \in Sel_s^t} u_{k's}^tn_{k'}}
\end{equation}
i.e., a client that has higher importance weight on cluster $s$ will be given higher aggregation weight, and vice versa. 
The introduction of the smoother $\sigma$ avoids the situation where $\sum_k u_{ks}^t = 0$ for some cluster, which could happen in the very beginning of the training when the center does not exhibit strength on any distributions. 
In that case, $v_{sk}^t = \frac{1}{N}$, i.e., the cluster will be updated in a manner that treats all clients equally. Otherwise, since $\sigma$ is very small, a client with $u_{ks}^t = \sigma$ will be assigned a $v_{sk}^t \approx 0$, and the aggregation weights of other clients will not be affected.

Though calculating and reporting $\{u_{ks}^t\}$ is computationally trivial compared to the actual training procedure, sending centers to all clients may introduce large communication costs. \textbf{FedSoft} thus allows the estimations of $u_{ks}$ to be used for up to $\tau \geq 1$ iterations (line 3). In practice, a client can start computing $u_{ks}^t$ for a cluster before it receives all other centers, the delay of transmission is thus tolerable.


Next, relevant clients run proximal local updates to find the minimizer $w_k^{t+1}$ for the proximal objective $h_k^t$, which entails solving only one optimization problem (line 17). In the case when all clients participate, the cluster models are produced by aggregating all client models: 
$ c_s^{t+1} = \sum\nolimits_{k=1}^N v_{sk}^t w_k^{t+1} $.
However, requiring full client participation is impractical in the federated setting. We thus use the client selection trick \cite{fedavg} to reduce the training cost (lines 16). For each cluster $s$, the algorithm randomly selects a small subset of clients $Sel_s^t$ to participate in the local updating at time $t$, where $|Sel_s^t| = K < N$. 


Clustered FL generally entails more clients to be selected compared to conventional FL, to ensure the convergence of all cluster models. Since \textbf{FedSoft} clients can contribute to multiple centers, however, we select only $|\cup_s Sel_s^t|$ clients instead of $\sum_s |Sel_s^t| = SK$ clients in the usual clustered FL. 
For example, if each distribution has the same share in every client, then in expectation only $N\left(1-\left(1-\frac{K}{N}\right)^S\right)$ clients will be selected. This number equals $2K-\frac{K^2}{N}$ when $S=2$, i.e., $\frac{K^2}{N}$ clients are selected by both clusters.

Once the server receives the local models $\{w_k^{t+1}\}$ for selected clients, it produces the next centers by simply averaging them (line 18). After completion, the algorithm yields trained cluster models $\{c_s^T\}$ as outputs, and each client obtains a personalized local model $w_k^T$ as a byproduct.

%% file: 4-analysis.tex
\section{Convergence Analysis} \label{sec:analysis}
In this section, we provide a convergence guarantee for \textbf{FedSoft}. First, note that we can rewrite (\ref{eq:hkt}) as follows:

\begin{equation}
    h_k(w_k;c^t) = \sum_{s, u_{ks} \neq 0} u_{ks} h_{ks}(w_k;c_s^t)
\end{equation}
\begin{equation}
    h_{ks}(w_k;c_s^t) \overset{\Delta}{=} F_s(w_k) + \frac{\lambda}{2}\frac{u_{ks}^t}{u_{ks}} \|w_k - c_s^t \|^2
\end{equation}

Here $h_{ks}$ is only defined for $u_{ks} \neq 0$, and we call optimizing each $h_{ks}$ a \emph{sub-problem} for client $k$. 

Our analysis relies on the following assumptions:

\begin{assumption} \label{as:inexact}
($\gamma_0$-inexact solution) Each client produces a $\gamma_0$-inexact solution $w_k^{t+1}$ for the local minimization of (\ref{eq:hkt}): 
\begin{equation}
    \| \nabla h_k(w_k^{t+1};c^t) \| \leq \gamma_0 \min_s \| \nabla F_s(c_s^t) \|
\end{equation}
\end{assumption}

\begin{assumption} \label{as:similarity}
($\beta$-similarity of sub-problems) The sub-problems $h_{ks}$ of each client $k$ have similar optimal points:
\begin{equation}
    \sum_{s'} u_{ks'} \|\nabla h_{ks'}(w_{ks}^*;c_s^t)\|^2 \leq \beta \|\nabla h_{ks}(c_s^t;c_s^t) \|^2, \forall s
\end{equation}
for some $\beta > 0$, where $w_{ks}^* = \textrm{argmin}_{w_{ks}} h_{ks}(w_{ks};c_s^t)$.
\end{assumption}

\begin{assumption} \label{as:convex_smooth}
(Strong convexity and smoothness) Cluster risks are $\mu_F$ strongly convex and $L_F$ smooth.
\end{assumption}



\begin{assumption} \label{as:init}
(Bounded initial error) At a certain time of the training, all centers have bounded distance from their optimal points. We begin our analysis at that point: 
\begin{equation}
    \|c_s^0 - c_s^* \| \leq \left(0.5 - \alpha_0 \right)\sqrt{\mu_F / L_F} \delta, \forall s
\end{equation}
where $0 < \alpha_0 \leq 0.5$.
\end{assumption}

Assumption~\ref{as:inexact} assumes significant progress is made on the proximal minimization of $h_k$, which is a natural extension from assumptions in \textbf{FedProx} \cite{fedprox}. Assumption~\ref{as:similarity} ensures the effectiveness of the joint optimization of $h_k$, i.e., solving one sub-problem can help identify the optimal points of others. Intuitively, if the sub-problems are highly divergent, we would not expect that solving them together would yield a universally good solution. This assumption quantifies our previous reasoning that different distributions co-existing in one local dataset have some similarities, which is the prerequisite for local models to converge and cluster models to be able to learn from each other. Assumption~\ref{as:convex_smooth} is standard \cite{hard_cluster_ucb}, 
and Assumption~\ref{as:init} is introduced by \citet{hard_cluster_ucb} in order to bound the estimation error of $u_{ks}^t$ (Lemma \ref{lm:p_epsilon}). 
Note that with Assumption \ref{as:convex_smooth}, each sub-problem $h_{ks}$ is also $\mu_\lambda$ strongly convex and $L_\lambda$ smooth, where $\mu_\lambda \geq \mu_F, L_\lambda \geq L_F$, and the subscript $\lambda$ indicates they increase with $\lambda$.

To measure the distance of different clusters, we quantify
\begin{equation}
    \delta \leq \|c_{s}^* - c_{s'}^* \| \leq \Delta, \forall s \neq s'
\end{equation}

As we will see later, soft clustered FL performs best when $\delta$ and $\Delta$ are close. Intuitively, a very small $\delta$ indicates two clusters are almost identical, and thus might be better combined into one distribution. On the other hand, a very large $\Delta$ implies that two clusters are too divergent, making it hard for one model to acquire useful knowledge from the other. 

Next, we bound $\mathbb{E}[u_{ks}^t]$ with respect to the true $u_{ks}$, for which we reply on the following lemma \cite{hard_cluster_ucb}:

\begin{lemma} \label{lm:p_epsilon} Suppose Assumptions \ref{as:convex_smooth} and \ref{as:init} hold. Denoting by $\mathcal{E}_t^{j,j'}$ the event that a data point $(x_j, y_j) \sim \mathcal{P}_j$ is incorrectly classified into cluster $j' \neq j$ at $t$, there exists a $c_\epsilon$ such that
\begin{equation}
    \mathbb{P}(\mathcal{E}_t^{j,j'}) \leq p_\epsilon \overset{\Delta}{=} \frac{c_\epsilon}{\alpha_0^2\delta^4} 
\end{equation}
\end{lemma}
Based on Lemma \ref{lm:p_epsilon}, we can bound $\mathbb{E}[u_{ks}^t]$ as follows 

\begin{theorem} \label{tm:E[ukst]} (Bounded estimation errors) The expectation of $u_{ks}^t$ is bounded as
\begin{equation}
    \mathbb{E}[u_{ks}^t] \leq (1-p_\epsilon)u_{ks} + p_\epsilon'
\end{equation}
Here $p_\epsilon' = p_\epsilon + \sigma$, and the expectation is taken over the randomness of samples.
\end{theorem}

Next, we seek to characterize each sub-problem $h_{ks}$ at the $\gamma_0$-inexact solution $w_k^{t+1}$ that approximately minimizes $h_k$.
Intuitively, $w_k^{t+1}$ should perform better for sub-problems with larger $u_{ks}$. On the other hand, if $u_{ks} = 0$, we generally cannot expect that $w_k^{t+1}$ will be close to $c_s^*$. We summarize this intuition in Theorems \ref{tm:sub_inexact} and \ref{tm:rDelta}.

\begin{theorem} \label{tm:sub_inexact} (Inexact solutions of sub-problems)
If $u_{ks} > 0$, and Assumptions \ref{as:inexact} to \ref{as:convex_smooth} hold, then
\begin{equation}
    \|\nabla h_{ks}(w_k^t; c_s^t) \| \leq \frac{\gamma}{\sqrt{u_{ks}}} \|\nabla F_s(c_s^t) \|
\end{equation}
where $\gamma = \sqrt{(\gamma_0^2 + \beta)L_\lambda / \mu_\lambda}$.
\end{theorem}

\begin{theorem} \label{tm:rDelta} (Divergence of local model to centers)
If Assumptions \ref{as:inexact}, \ref{as:convex_smooth}, and \ref{as:init} hold, we have
\begin{equation}
    \|w_k^{t+1} - c_s^t \| \leq r \Delta, \forall s
\end{equation}
where $r = \frac{\gamma S + 1}{4} \sqrt{\frac{L_F}{\mu_F}} + \frac{1}{2} \sqrt{\frac{\mu_F}{L_F}} + 1$.
\end{theorem}

Theorem \ref{tm:sub_inexact} indicates that if the $h_k$ is solved with high quality $w_k^{t+1}$ (small $\gamma_0$), and the sub-problems are sufficiently similar (small $\beta$), then sub-problems with $u_{ks} > 0$ can also be well solved by $w_k^{t+1}$. It also justifies using $v_{sk}^t \propto u_{ks}^t$ as aggregation weights in (\ref{eq:vskt}). In the case $u_{ks} = 0$, according to Theorem~\ref{tm:rDelta} (which holds for any $u_{ks}$), approaching $c_s^t$ with $w_k^t$ will introduce an error of at most $O(\Delta)$.

Finally, we show the convergence of $F_s(c_s^t)$. The following analysis does not depend on $\tau$; 
we show how $\tau$ affects the convergence in Appendix \ref{sec:appendix_tau}.

\begin{theorem} \label{tm:convergence} (Convergence of centers)
    Suppose Assumptions \ref{as:inexact} to \ref{as:init} hold, and define the quantities: $n \overset{\Delta}{=} \sum_k n_k, n_s \overset{\Delta}{=} \sum_{k} u_{ks}n_k, m_s \overset{\Delta}{=} \sum_{k, u_{ks} \neq 0}n_k, \Bar{m}_s \overset{\Delta}{=} \sum_{k, u_{ks} = 0} n_k, \Hat{m}_s \overset{\Delta}{=} (1-p_\epsilon)m_s + p_\epsilon'\sum_{k, u_{ks} \neq 0} \frac{n_k}{u_{ks}}$. Suppose $\lambda$ is chosen such that $\rho \overset{\Delta}{=} \frac{n_s - \gamma m_s}{\lambda} - \frac{(\gamma + 1)L_F m_s}{\mu_\lambda \lambda} - \frac{p_\epsilon' \Bar{m}_s}{2\mu_\lambda} - \frac{L_F(\gamma+1)^2\hat{m}_s}{2\mu_\lambda^2} - \frac{4L_F(\gamma+1)^2\hat{m}_s}{\mu_\lambda^2\sqrt{K}} - \frac{(\gamma+1)^2\hat{m}_s + (1-p_\epsilon)n_s + p_\epsilon' n}{\mu_\lambda\sqrt{2K}} > 0$ and denote $R \overset{\Delta}{=} \frac{1}{2}\left(\mu_\lambda + L_F \right)\Bar{m}_s r^2 + \frac{(4L_F + \mu_\lambda)\Bar{m}_sr^2}{\sqrt{K}}$. Then we have
    \begin{equation}
    \begin{split}
        &\mathbb{E}[F_s(c_s^{t+1})] - F_s(c_s^t) \\
        \leq& -  \frac{\rho \|\nabla F_s(c_s^t)\|^2}{(1-p_\epsilon)n_s + p_\epsilon' n} + \frac{p_\epsilon' R \Delta^2}{(1-p_\epsilon)n_s - p_\epsilon'(S-2) n}
    \end{split}
    \end{equation}
    at any time $t$, where the expectation is taken over the selection of clients and all $\{u_{ks}^t\}$.
\end{theorem}

\begin{corollary} \label{coro:convergence}
Suppose $F_s(c_s^0) - F_s^* = B_s$. After $T$ iterations,
\begin{equation}
\begin{split}
    &\sum_{t=1}^T \frac{\rho \mathbb{E}\|\nabla F_s(c_s^t)\|^2}{(1-p_\epsilon)n_s + p_\epsilon' n} \leq \frac{B_s}{T} + O(p_\epsilon \Delta^2)
\end{split}
\end{equation}
\end{corollary}

The choices of $\lambda$ to make $\rho > 0$ is discussed in \cite{fedprox}. From Corollary \ref{coro:convergence}, the gradient norm converges to a remaining error controlled by $p_\epsilon$. Intuitively, when $c_s^t = c_s^*$, further updating $c_s^t$ with misclassified models will inevitably move $c_s^t$ away from $c_s^*$. This bias cannot be removed unless we have a perfect guess of $u_{ks}$. Recall that $p_\epsilon = O(\frac{1}{\delta^4})$, and thus the remaining term is $O(\frac{\Delta^2}{\delta^4})$, which decreases as $\Delta$ approaches $\delta$. Thus, \textbf{FedSoft} performs better when the divergences between clusters are more homogeneous. Note that Corollary~\ref{coro:convergence} seems to imply the remaining error will explode if $\delta \rightarrow 0$, but Lemma~\ref{lm:p_epsilon} is only valid when $p_\epsilon < 1$. Thus when $\delta$ is very small, i.e., there exist two distributions that are extremely similar, the remaining error is determined by the maximum divergence of the other distributions. Furthermore, the divergence $\Delta$ determines the degree of non-IID of a local dataset (not among clients), which also implicitly affects the accuracy of local solutions $\gamma_0$. Intuitively, a larger $\Delta$ implies it is more difficult to exactly solve a local problem involving multiple distributions, resulting in a greater $\gamma_0$.

To see the role of cluster heterogeneity, suppose $\|c_1^* - c_2^*\|$ is closer than the distance of all other centers to $c_1$, then the misclassified samples for cluster 1 are more likely to be matched to cluster 2. Thus, cluster 2 gets more updates from data that it does not own, producing greater remaining training error that drags its center towards cluster 1. On the other hand, if the cluster divergence is homogeneous, then the effect of mis-classification is amortized among all clusters, resulting in a universally smaller remaining error.

Theorem \ref{tm:convergence} shows the convergence of cluster models $\{c_s\}$ in terms of the cluster risks $\{F_s\}$. For the local models $\{w_k\}$, we focus on how clients integrate global knowledge into their local personalizations, which cannot be captured only with the original client risk functions $\{f_k(w)\}$. 
Thus, we are interested in the convergence performance of $\{w_k\}$ with respect to the proximal objective $\{h_k\}$. Note that under Assumption \ref{as:convex_smooth}, \textbf{FedSoft} is effectively a cyclic block coordinate descent algorithm on a jointly convex objective function of $\{w_k\}$ and $\{c_s\}$, for which the convergence is guaranteed: 

\begin{theorem} \label{tm:joint_convergence} (Joint convergence of cluster and client models)
    For fixed importance weights $\Tilde{u}$, let $w^*,c^* = \text{argmin} \sum_{k=1}^N h_k(w_k;c,\Tilde{u}_k)$, and $w^T,c^T$ be the outputs of \textbf{FedSoft}. Then $w^T \rightarrow w^*, c^T \rightarrow c^*$ linearly with $T$.
\end{theorem}

%% file: 5-experiments.tex
\section{Experiments} \label{sec:exp}
In this section, we verify the effectiveness of \textbf{FedSoft} with two base datasets under various mixture patterns. For all experiments, we use $N=100$ clients, and the number of samples in each client $n_k$ is chosen uniformly at random from 100 to 200. For ease of demonstration, for every base dataset, we first investigate the mixture of $S=2$ distributions and then increase $S$. In the case with two distributions, suppose the cluster distributions are named $\mathcal{D}_A$ and $\mathcal{D}_B$. We evaluate the following partition patterns:

\begin{itemize}
    \item 10:90 partition: 50 clients have a mixture of 10\% $\mathcal{D}_A$ and 90\% $\mathcal{D}_B$, and 50 have 10\% $\mathcal{D}_B$ and 90\% $\mathcal{D}_A$.
    \item 30:70 partition: Same as above except the ratio is 30:70.
    \item Linear partition: Client $k$ has $(0.5+k)$\% data from $\mathcal{D}_A$ and $(99.5-k)$\% data from $\mathcal{D}_B$, $k = 0,\cdots,99$.
\end{itemize}

We further introduce the random partition, where each client has a random mixture vector generated by dividing the $[0,1]$ range into $S$ segments with $S-1$ points drawn from $\text{Uniform}(0,1)$. We use all four partitions for $S=2$, and only use the random partition when $S>2$ for simplification. Each partition produces non-IID local distributions, i.e., clients have different local data distributions. Specifically, the 10:90 and 30:70 partitions yield 2 local distributions, while the linear and random partitions yield 100.  
Unless otherwise noted, we choose \textbf{FedSoft}'s estimation interval $\tau=2$, client selection size $K=60$, counter smoother $\sigma$ = 1e-4, and all experiments are run until both cluster and client models have fully converged. All models are randomly initialized with the Xavier normal \cite{xavier} initializer without pre-training, so that the association among clients, centers, and cluster distributions is built automatically during the training process. 


We compare \textbf{FedSoft} with two baselines: \textbf{IFCA} \cite{hard_cluster_ucb} and \textbf{FedEM} \cite{fmtl_mixture}. Both baseline algorithms produce one center for each cluster, but they do not explicitly generate local models as in \textbf{FedSoft}. Nevertheless, they also estimate the importance weights for each client, we thus use the center corresponding to the largest importance weight as a client's local model. Since we expect cluster models will be deployed to new users, we evaluate their test accuracy/error on holdout datasets sampled from the corresponding cluster distributions. For local models, they are expected to fit the local data of participating clients, we hence evaluate their accuracy/error on local training datasets. Throughout this section, we use $\Bar{c}$ and $\Bar{w}$ to represent the average accuracy/error of the cluster and client models, not the accuracy/error of the averaged models.

\begin{figure}[t]
\includegraphics[width=8cm]{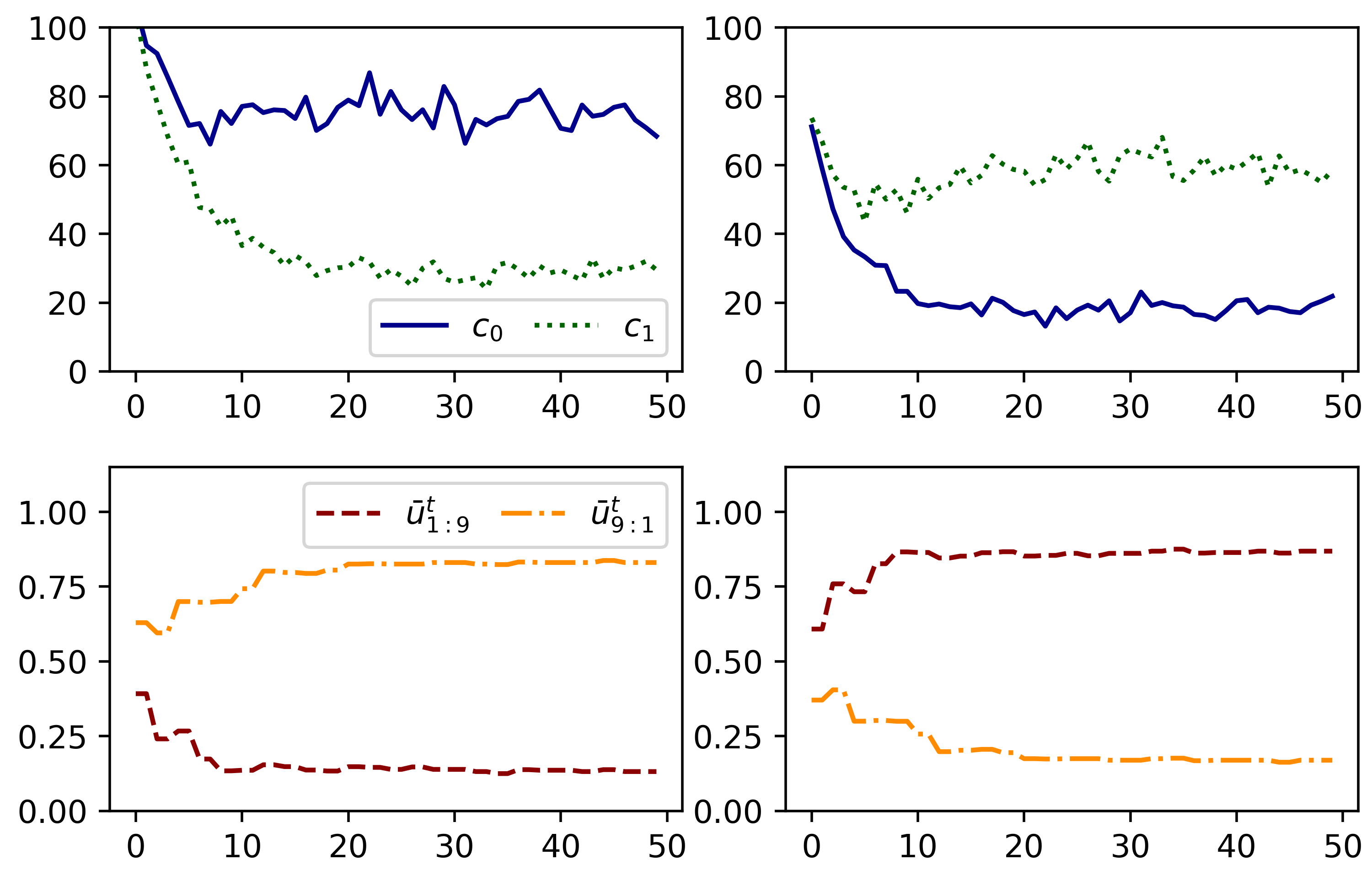}
\centering
\caption{Evolution of the test mean squared error of centers (top) and the importance weight estimation of clients (bottom) over time for the mixture of two synthetic distributions under the 10:90 partition. The left/right columns represent the first/second distributions. Center indices are assigned randomly in the beginning. The importance weight estimations $\Bar{u}_{a:b}^t$ are averaged on clients with the mixture coefficients $a:b$ (i.e., they have the same local distribution).\vspace{-0.8em}} \label{fig:synthetic_convergence}
\end{figure}

We use three datasets to generate the various distributions: Synthetic, EMNIST and CIFAR-10. Due to the space limit, we put the details of experiment parameters and CIFAR-10 results in Appendix \ref{sec:appendix_exp}.
\begin{itemize}
    \item \textbf{Synthetic Data}. We generate synthetic datasets according to $y_i = \langle x_i, \theta_s \rangle + \epsilon_i$ where $\theta_s \sim \mathcal{N}(0, \sigma_0^2 I_{10})$, $x_i \sim \mathcal{N}(0, I_{10})$, $\epsilon_i \sim \mathcal{N}(0, 1)$ \cite{hard_cluster_ucb}. Unless otherwise noted, we use $\sigma_0=10$. We use the conventional linear regression model for this dataset. 
    \item \textbf{EMNIST Letters}. We use the handwritten images of English letters in the EMNIST dataset to create 2 distributions for the lower and uppercase letters \cite{oneshot}, each with 26 classes. Then we rotate these images counterclockwise by $90\degree$ \cite{mnist_rotation}, resulting in 4 total distributions. In the $S=2$ setting we compare the two $0\degree$ distributions. A rotation variant CNN model is used for this dataset.
\end{itemize}

In general, the letter distributions share more similarities with each other, while the synthetic distributions are more divergent, e.g., letters like ``O'' have very similar upper and lowercase shapes and are invariant to rotations. 
On the other hand, data generated from $y=x$ and $y=-x$ can be easily distinguished. We thus expect the mixture of synthetic data to benefit more from the personalization ability of \textbf{FedSoft}.

The typical convergence process of \textbf{FedSoft} is shown in Figure \ref{fig:synthetic_convergence}. In this example of the synthetic data, \textbf{FedSoft} is able to automatically distinguish the two cluster distributions. After around 5 global epochs, center 1 starts to exhibit strength on the first cluster distribution, and center 0 concentrates on the other, which implies a correct association between centers and cluster distributions. Similarly, the importance weight estimations $u_{ks}^t$, which are initially around 50:50, soon converge to the real mixture ratio 10:90.

\begin{table}[t]
\centering
\caption{MSE or accuracy of cluster models for the mixture of two distributions. Each row represents the distribution of a test dataset. The center with the smallest error or highest accuracy is underlined for each test distribution.} \label{tab:centers_2partition}
\textbf{Synthetic data: mean squared error}
\vspace{0.3em}

\addtolength{\tabcolsep}{-1.2pt}
\begin{tabular}{ccccccccc}
\Xhline{2\arrayrulewidth}
\multirow{2}{*}{} & \multicolumn{2}{c}{10:90} & \multicolumn{2}{c}{30:70} & \multicolumn{2}{c}{Linear} & \multicolumn{2}{c}{Random} \\ \cline{2-9}
                  & $c_0$       & $c_1$       & $c_0$       & $c_1$       & $c_0$        & $c_1$       & $c_0$        & $c_1$       \\ \hline
$\theta_0$           & 68.4       & \underline{29.5}       & \underline{44.2}       & 49.6       & \underline{38.2}        & 59.1      & \underline{42.2}        & 60.6       \\
$\theta_1$          & \underline{21.8}       & 58.6       & 41.3       & \underline{36.3}       & 47.1        & \underline{27.8}      & 42.7        & \underline{27.0}       \\ \Xhline{2\arrayrulewidth}
\end{tabular}

\vspace{0.5em}

\textbf{EMNIST letters: accuracy (\%)}
\vspace{0.3em}

\addtolength{\tabcolsep}{-.8pt}
\begin{tabular}{ccccccccc}
\Xhline{2\arrayrulewidth}
\multirow{2}{*}{} & \multicolumn{2}{c}{10:90} & \multicolumn{2}{c}{30:70} & \multicolumn{2}{c}{Linear} & \multicolumn{2}{c}{Random} \\ \cline{2-9}
                  & $c_0$       & $c_1$       & $c_0$       & $c_1$       & $c_0$        & $c_1$       & $c_0$        & $c_1$       \\ \hline
Lower           & 68.9       & \underline{70.3}       & \underline{65.9}       & 65.8       & \underline{71.8}        & 71.7      & 72.0        & \underline{72.5}       \\
Upper          & \underline{74.6}       & 73.3       & 70.1       & \underline{70.4}       & 73.9        & \underline{74.1}      & \underline{77.7}        & 77.2       \\ \Xhline{2\arrayrulewidth}
\end{tabular}




\end{table}




\begin{table}[t]
\centering
\caption{Comparison between FedSoft and baselines on the letters data. $c_{lo}^*$/$c_{up}^*$ represents the accuracy of the center that performs best on the lower/upper distribution, and the number in the parenthesis indicates the index of that center. $\Bar{w}$ is the accuracy of local models averaged over all clients.} \label{tab:two_mixture_comparison}
\addtolength{\tabcolsep}{-2pt}
\begin{tabular}{ccccccc}
\Xhline{2\arrayrulewidth}
\multirow{2}{*}{} & \multicolumn{3}{c}{10:90} & \multicolumn{3}{c}{Linear} \\ \cline{2-7}
                  & $c_{lo}^*$       & $c_{up}^*$       & $\Bar{w}$       & $c_{lo}^*$       & $c_{up}^*$        & $\Bar{w}$       \\ \hline
FedSoft           & 70.3(1)       & 74.6(0)       & 90.9       & 71.8(0)       & 74.1(1)        & 86.5       \\
IFCA          & 58.5(0)       & 61.3(1)       & 65.2       & 55.4(0)       & 57.2(0)        & 62.9       \\
FedEM          & 67.4(1)       & 69.8(1)      & 63.6  & 65.9(0)       & 69.0(0)        & 62.4       \\\Xhline{2\arrayrulewidth}
\end{tabular}
\vspace{-1em}
\end{table}

\begin{figure}[t]
\includegraphics[width=8cm]{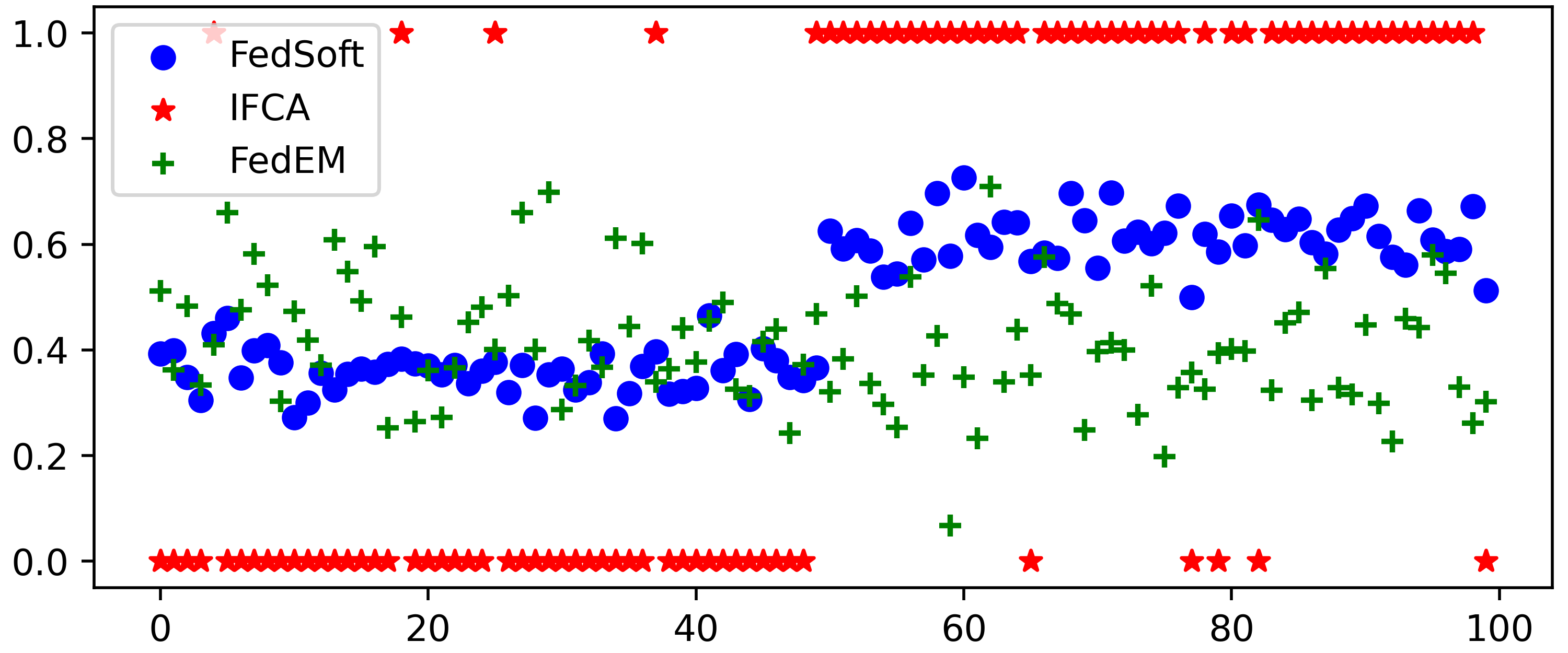}
\centering
\vspace{-0.8em}
\caption{The clients' estimation of importance weights on the first cluster ($u_{k0}^T$) under the 10:90 partition of the EMNIST dataset. The X axis is the index of clients and each point represents a client.} \label{fig:client_importance_estimation}
\end{figure}

Table \ref{tab:centers_2partition} lists the mean squared error (MSE) or accuracy of the output cluster models. \textbf{FedSoft} produces high quality centers under all mixture patterns. In particular, each center exhibits strength on one distribution, which indicates that \textbf{FedSoft} builds correct associations for the centers and cluster distributions. The performance gap between two distributions using the same center is larger for the synthetic data. This is because the letter distributions have smaller divergence than the synthetic distributions. Thus, letter models can more easily transfer the knowledge of one distribution to another, and a center focusing on one distribution can perform well on the other. Notably, the 30:70 mixture has the worst performance for both datasets, which is due to the degrading of local solvers when neither distribution dominates. Thus, the local problems under this partition are solved less accurately, resulting in poor local models and a large value of $\gamma$ in Theorem~\ref{tm:sub_inexact}, which then produces high training loss on cluster models according to Theorem \ref{tm:convergence}.

Table \ref{tab:two_mixture_comparison} compares \textbf{FedSoft} with the baselines. Not only does \textbf{FedSoft} produce more accurate cluster and local models, but it also achieves better balance between the two trained centers. Similarly, Figure 2 shows the importance estimation of clients for the first cluster. \textbf{FedSoft} and \textbf{IFCA} are able to build the correct association (though the latter is a hard partition), while \textbf{FedEM} appears to be biased to the other center by putting less weights ($<0.5$) on the first one.

Next, we evaluate the algorithm with the random partition for the mixture of more distributions. Tables \ref{tab:synthetic_mixture8} and \ref{tab:letters_mixture4} show the MSE or accuracy of cluster models for the mixture of 8 and 4 distributions on synthetic and letters data, where we still observe high-quality outcomes and good association between centers and cluster distributions.

\begin{table}[h]
\centering
\caption{Test MSE of centers for the mixture of 8 synthetic distributions with randomly generated weights $\theta_0 \cdots \theta_7$.} \label{tab:synthetic_mixture8}

\addtolength{\tabcolsep}{-3.2pt}
\label{tab:mnist_4cluster_accuracy}
\begin{tabular}{ccccccccc}
\Xhline{2\arrayrulewidth}
    & $c_0$ & $c_1$ & $c_2$ & $c_3$ & $c_4$ & $c_5$ & $c_6$ & $c_7$ \\ \hline
$\theta_0$   & 62.2 & 62.7 & 64.0 &  63.2 & 61.4 & \underline{57.6} & 63.7 & 61.9         \\
$\theta_1$   & 65.0 &  67.8 & 69.4 &  65.8 & 64.3 & 67.1 & \underline{64.2} & 64.5         \\
$\theta_2$   & 60.1 & 59.9 & \underline{57.8} & 58.6 & 62.6 & 63.2 & 60.0 & 59.9         \\
$\theta_3$   & 96.8 & 95.4 & 96.8 & 98.8 & \underline{93.2} & 93.8 & 95.3 & 96.8         \\
$\theta_4$   & 86.1 & 89.8 & 91.8 & 87.3 & 85.4 & 86.6 & 85.6 & \underline{84.9}        \\
$\theta_5$   & 161.2 & 160.3 & \underline{156.0} & 160.0 & 164.7 & 167.3 & 162.0 & 163.6 \\
$\theta_6$   & 110.2 & 106.7 & \underline{104.8} & 109.0 & 111.0 & 107.8 & 111.5 & 110.1         \\
$\theta_7$   & 34.5 & \underline{33.8} & 34.8 & 33.9 & 34.2 & 34.1 & 34.4 & 35.0  \\ \Xhline{2\arrayrulewidth}
\end{tabular}
\end{table}

\vspace{-1em}
\begin{table}[h]
\centering
\caption{Test accuracy (\%) of centers for the mixture of 4 distributions with original and $90\degree$-rotated letter images.} \label{tab:letters_mixture4}
\addtolength{\tabcolsep}{-2.5pt}
\begin{tabular}{ccccccccc}
\Xhline{2\arrayrulewidth}
\multirow{2}{*}{} & \multicolumn{2}{c}{$c_0$} & \multicolumn{2}{c}{$c_1$} & \multicolumn{2}{c}{$c_2$} & \multicolumn{2}{c}{$c_3$} \\ \cline{2-9}
                  & $0\degree$       & $90\degree$       & $0\degree$       & $90\degree$       & $0\degree$        & $90\degree$       & $0\degree$        & $90\degree$       \\ \hline
Lower           & 71.5       & \underline{67.6}       & 71.3       & 67.3       & 71.3        & \underline{67.6}      & \underline{72.3}        & 67.3       \\
Upper          & 70.2       & 71.7       & \underline{70.8}       & 71.3       & 70.3       & \underline{71.9}      & 70.3       & 71.0       \\ \Xhline{2\arrayrulewidth}
\end{tabular}
\end{table}


%% file: 6-conclusion.tex
\section{Conclusion} \label{sec:conclusion}
This paper proposes \textbf{FedSoft}, an efficient algorithm generalizing traditional clustered federated learning approaches to allow clients to sample data from a mixture of distributions. By incorporating proximal local updating, \textbf{FedSoft} enables simultaneous training of cluster models for future users, and personalized local models for participating clients, which is achieved without increasing the workload of clients. Theoretical analysis shows the convergence of \textbf{FedSoft} for both cluster and client models, and the algorithm exhibits good performance in experiments with various mixture patterns.

%% file: 7-appendix.tex
\appendix
\onecolumn
\section{Proof of Theorems}
\subsection{Proof of Theorem \ref{tm:E[ukst]}} 
Let $G_{ki}, i=1,2\cdots$ be some virtual group of client $k$'s data points that follow the same distribution. Thus,
\begin{proof}
\begin{equation}
\begin{split}
    \mathbb{E} [u_{ks}^t] \leq& \frac{1}{n_k} \left(\sum_{G_{ki} \sim \mathcal{P}_s} \mathbb{P}\left(\text{argmin}_{s'}F_{s'}(G_{ki}) = s\right) |G_{ki}| + \sum_{G_{ki} \not\sim \mathcal{P}_{s}} \mathbb{P}\left(\text{argmin}_{s'}F_{s'}(G_{ki}) = s\right) |G_{ki}|\right) + \sigma \\
    \leq& \frac{1}{n_k} \left(n_{ks} + \sum_{G_{ki} \not\sim \mathcal{P}_{s}} p_\epsilon |G_{ki}|\right) = \frac{1}{n_k} \left(n_{ks} + p_\epsilon (n_k - n_{ks})\right) + \sigma = (1-p_\epsilon)u_{ks} + p_\epsilon'
\end{split}
\end{equation}
\end{proof}

\subsection{Proof of Theorem \ref{tm:sub_inexact}}
\begin{proof}
For simplification we drop the dependency of $c_s^t$ in $h_k$ and $h_{ks}$. Take $w_{ks}^* = \text{argmin}_{w_{ks}} h_{ks}(w_{ks})$, we have
\begin{equation}
\begin{split}
    &\gamma_0^2 \| \nabla F_s(c_s^t) \|^2 \geq \|\nabla h_k(w_k^t) \|^2 \geq 2\mu_\lambda \left(h_k(w_k^t) - h_k(w_{ks}^*)\right) \\
    =& 2\mu_\lambda \sum_{s'} u_{ks'} \left(h_{ks'}(w_k^t) - h_{ks'}(w_{ks}^*)\right) = 2\mu_\lambda u_{ks} \left(h_{ks}(w_k^t) - h_{ks}^*\right) - 2\mu_\lambda  \sum_{s' \neq s} u_{ks'} \left(h_{ks'}(w_{ks}^*) - h_{ks'}(w_k^t)\right) \\
    \geq& \frac{\mu_\lambda}{L_\lambda} u_{ks} \|\nabla h_{ks}(w_k^t) \|^2 - \sum_{s' \neq s} u_{ks'} \|\nabla h_{ks'}(w_{ks}^*) \|^2 \geq \frac{\mu_\lambda}{L_\lambda} u_{ks} \|\nabla h_{ks}(w_k^t) \|^2 - \beta \|\nabla F_s(c_s^t) \|^2
\end{split}
\end{equation}

Thus,

\begin{equation}
    \frac{\mu_\lambda}{L_\lambda} u_{ks} \|\nabla h_{ks}(w_k^t) \|^2  \leq \left(\gamma_0^2 + \beta\right) \|\nabla F_s(c_s^t) \|^2
\end{equation}

Reorganizing, we have
\begin{equation}
    \|\nabla h_{ks}(w_k^t)\| \leq \frac{\sqrt{(\gamma_0^2 + \beta)L_\lambda / \mu_\lambda}}{\sqrt{u_{ks}}} \|\nabla F_s(c_s^t)\| = \frac{\gamma}{\sqrt{u_{ks}}} \|\nabla F_s(c_s^t)\|
\end{equation}
\end{proof}

\subsection{Proof of Theorem \ref{tm:rDelta}}
\begin{proof}
Let $s' = \text{argmax}_s u_{ks}$, we have $u_{ks'} \geq \frac{1}{S}$, thus
\begin{equation}
\begin{split}
     &\|w_k^{t+1} - c_s^t \| = \|w_k^{t+1} - c_{s'}^t + c_{s'}^t - c_s^t \| \\
     \leq& \frac{1}{\mu_\lambda} \|\nabla h_{ks'}(w_k^{t+1};c_{s'}^t) - \nabla h_{ks'}(c_{s'}^t;c_{s'}^t) \| + \|c_{s'}^t - c_s^t \| \\
     \leq& \frac{1}{\mu_F} \left(\frac{\gamma}{u_{ks'}}\|\nabla F_{s'}(c_{s'}^t) \| + \|\nabla F_{s'}(c_{s'}^t) \|\right) + \frac{1}{2} \sqrt{\frac{\mu_F}{L_F}} \delta + \Delta \\
     \leq& \frac{(\gamma S + 1)L_F}{\mu_F} \|c_{s'}^t - c_{s'}^* \| + \left(\frac{1}{2} \sqrt{\frac{\mu_F}{L_F}} + 1\right)\Delta \\
     \leq& \left(\frac{\gamma S + 1}{4} \sqrt{\frac{L_F}{\mu_F}} + \frac{1}{2} \sqrt{\frac{\mu_F}{L_F}} + 1\right) \Delta
\end{split}
\end{equation}
\end{proof}

\subsection{Proof of Theorem \ref{tm:convergence}}
We first introduce the following lemma:

\begin{lemma} \label{lm:exp_ind}
$\mathbb{E}\left[\frac{u_{ks}^t n_k}{\sum_{k'}u_{k's}^t n_{k'}}\right] \leq \mathbb{E}[u_{ks}^tn_k]\mathbb{E} \left[\frac{1}{\sum_{k'} u_{k's}^tn_{k'}}\right]$, where the expectation is taken over $\{u_{ks}^t\}$.
\end{lemma}

\begin{proof}
Let $r^t = \sum_{k' \neq k} u_{k's}^tn_{k'}$, and note that $u_{ks}^t \perp r^t$ since each client estimates its $u_{ks}^t$ independently. Thus,

\begin{equation}
\begin{split}
    &\mathbb{E}\left[\frac{u_{ks}^tn_k}{\sum_{k'}u_{k's}^tn_{k'}}\right] = \mathbb{E}\left[\frac{u_{ks}^tn_k}{u_{ks}^tn_k + r^t}\right] = \mathbb{E}\left[1 - \frac{r^t}{u_{ks}^tn_k + r^t}\right] \\
    =& \mathbb{E}_{r^t}\left[\mathbb{E}_{\{u_{ks}^t\}|r^t}\left[1 - \frac{r^t}{u_{ks}^tn_k + r^t}\right] \right] \leq \mathbb{E}_{r^t}\left[1 - \frac{r^t}{\mathbb{E}[u_{ks}^tn_k|r^t] + r^t}\right] \\
    =& \mathbb{E}[u_{ks}^tn_k]\mathbb{E}_{r^t}\left[\frac{1}{\mathbb{E}[u_{ks}^tn_k] + r^t}\right] \leq \mathbb{E}[u_{ks}^tn_k]\mathbb{E}_{r^t}\left[\mathbb{E}_{\{u_{ks}^t\}|r^t}\left[\frac{1}{u_{ks}^tn_k + r^t}\right]\right] \\
    =& \mathbb{E}[u_{ks}^tn_k]\mathbb{E}\left[\frac{1}{\sum_{k'}u_{k's}^tn_{k'} }\right]
\end{split}
\end{equation}
\end{proof}

We then formally prove Theorem \ref{tm:convergence}:
\begin{proof}
For $u_{ks} \neq 0$, define
\begin{equation}
    e_{ks}^{t+1} \overset{\Delta}{=} \nabla h_{ks}(w_k^{t+1};c_s^t) = \nabla F_s(w_k^{t+1}) + \lambda \frac{u_{ks}^t}{u_{ks}} (w_k^{t+1} - c_s^t)
\end{equation}

using Theorem \ref{tm:sub_inexact}, we have
\begin{equation}
    \|e_{ks}^{t+1}\| \leq \frac{\gamma}{\sqrt{u_{ks}}} \|\nabla F_s(c_s^t) \| \leq \frac{\gamma}{u_{ks}} \|\nabla F_s(c_s^t) \|
\end{equation}

Define
\begin{equation}
    \Bar{c}_s^{t+1} \overset{\Delta}{=} \sum_k v_{sk}^t w_k^{t+1} = \mathbb{E}_{v_{sk}^t}[w_k^{t+1}]
\end{equation}

thus,
\begin{equation}
\begin{split}
    \Bar{c}_s^{t+1} - c_s^t =& -\frac{1}{\lambda} \sum_{k, u_{ks} \neq 0} \left(\frac{u_{ks}n_k}{\sum_{k'} u_{k's}^tn_{k'}} \nabla F_s(w_k^{t+1}) - \frac{u_{ks}n_k}{\sum_{k'} u_{k's}^tn_{k'}} e_{ks}^{t+1} \right) + \sum_{k, u_{ks} = 0} v_{sk}^t (w_{k}^{t+1} - c_s^t)
\end{split}
\end{equation}

Next we bound $\|w_k^{t+1} - c_s^t \|$ for $u_{ks} \neq 0$,
\begin{equation}
\begin{split}
    &\|w_k^{t+1} - c_s^t\| \leq \frac{1}{\mu_\lambda} \|\nabla h_{ks}(w_k^{t+1};c_s^t) - \nabla h_{ks}(c_s^t;c_s^t) \| \\
    \leq& \left(\frac{\gamma}{\mu_\lambda \sqrt{u_{ks}}} + \frac{1}{\mu_\lambda} \right) \| \nabla F_s(c_s^t)\| \leq \frac{\gamma + 1}{\mu_\lambda \sqrt{u_{ks}}}\|\nabla F_s(c_s^t)\| \leq \frac{\gamma + 1}{\mu_\lambda u_{ks}}\|\nabla F_s(c_s^t)\|
\end{split}
\end{equation}

Therefore,
\begin{equation} \label{eq:|cs_-cs|}
\begin{split}
    &\|\Bar{c}_s^{t+1} - c_s^t\|^2 \leq \mathbb{E}_{v_{sk}^t}[\|w_k^{t+1} - c_s^t\|^2]\\
    \leq& \left(\frac{\gamma + 1}{\mu_\lambda}\right)^2 \sum_{k, u_{ks} \neq 0} \frac{u_{ks}^tn_k}{u_{ks} \sum_{k'} u_{k's}^tn_{k'}} \|\nabla F_s(c_s^t)\|^2  + \sum_{k, u_{ks = 0}} \frac{ u_{ks}^t n_k}{\sum_{k'} u_{k's}^tn_{k'}} r^2 \Delta^2
\end{split}
\end{equation}

Define $M_{s}^{t+1}$ such that $\Bar{c}_s^{t+1} - c_s^t = -\frac{1}{\lambda}\left(\frac{\sum_k u_{ks}n_k}{\sum_k u_{ks}^tn_k} \nabla F_s(c_s^t) + M_s^{t+1}\right)$

\begin{equation}
\begin{split}
    M_s^{t+1} =& \sum_{k, u_{ks} \neq 0} \left(\frac{u_{ks}n_k}{\sum_{k'} u_{k's}^tn_{k'}} \nabla F_s(w_k^{t+1}) - \frac{u_{ks}n_k}{\sum_{k'} u_{k's}^tn_{k'}} e_{ks}^{t+1} \right) - \frac{\sum_k u_{ks}n_k}{\sum_k u_{ks}^tn_k}\nabla F_s(c_s^t) - \lambda\sum_{k, u_{ks} = 0} v_{sk}^t (w_k^{t+1} - c_s^t ) \\
    =& \frac{1}{\sum_k u_{ks}^tn_k} \sum_{k, u_{ks} \neq 0} \left( u_{ks}n_k \left(\nabla F_s(w_k^{t+1}) - \nabla F_s(c_s^t)\right) - u_{ks}n_k e_{ks}^{t+1} \right) - \lambda\sum_{k, u_{ks} = 0} v_{sk}^t (w_k^{t+1} - c_s^t )
\end{split}
\end{equation}

thus,
\begin{equation}
\begin{split}
    &\| M_{t+1} \| \leq \frac{1}{\sum_k u_{ks}^tn_k} \sum_{k, u_{ks} \neq 0} \left(u_{ks}n_kL_F\|w_k^{t+1} - c_s^t\| + u_{ks}n_k \|e_{ks}^{t+1}\|\right) + \lambda \sum_{k, u_{ks} = 0} v_{sk}^t \|w_k^{t+1} - c_s^t \| \\
    \leq& \frac{m_s}{\sum_k u_{ks}^tn_k} \left(\frac{(\gamma + 1)L_F}{\mu_\lambda} + \gamma\right)\|\nabla F_s(c_s^t)\| + \lambda\sum_{k, u_{ks = 0}} \frac{ u_{ks}^tn_k}{\sum_{k'} u_{k's}^tn_{k'}} r \Delta
\end{split}
\end{equation}

Using the smoothness of $F_s$, we have
\begin{equation}
\begin{split}
    &F_s(\Bar{c}_s^{t+1}) \leq F_s(c_s^t) + \langle\nabla F_s(c_s^t), \Bar{c}_s^{t+1} - c_s^t \rangle + \frac{L_F}{2} \|\Bar{c}_s^{t+1} - c_s^t \|^2 \\
    \leq& F_s(c_s^t) - \frac{1}{\lambda} \frac{\sum_k u_{ks}n_k}{\sum_k u_{ks}^tn_k} \|\nabla F_s(c_s^t)\|^2 - \frac{1}{\lambda} \langle\nabla F_s(c_s^t), M_s^{t+1} \rangle + \frac{L_F}{2} \|\Bar{c}_s^{t+1} - c_s^t \|^2 \\
    \leq& F_s(c_s^t) - \frac{1}{\lambda} \frac{n_s}{\sum_k u_{ks}^tn_k} \|\nabla F_s(c_s^t)\|^2 + \frac{m_s}{\lambda \sum_k u_{ks}^tn_k} \left(\frac{ (\gamma + 1)L_F}{\mu_\lambda} + \gamma \right)\|\nabla F_s(c_s^t)\|^2 \\
    +& \left(\sum_{k, u_{ks = 0}} \frac{u_{ks}^tn_k}{\sum_{k'} u_{k's}^tn_{k'}} r \Delta\right)\|\nabla F_s(c_s^t)\| + \frac{L_F}{2}\left(\frac{\gamma + 1}{\mu_\lambda}\right)^2 \sum_{k, u_{ks} \neq 0} \frac{u_{ks}^tn_k}{u_{ks} \sum_{k'} u_{k's}^tn_{k'}} \|\nabla F_s(c_s^t)\|^2 \\
    +& \frac{L_F}{2}\sum_{k, u_{ks = 0}} \frac{ u_{ks}^tn_k}{\sum_{k'} u_{k's}^tn_{k'}} r^2 \Delta^2 \\
    =& F_s(c_s^t) - \frac{1}{\sum_k u_{ks}^tn_k}\left(\frac{n_s - \gamma m_s}{\lambda} - \frac{(\gamma + 1)L_F m_s}{\mu_\lambda \lambda}\right)\|\nabla F_s(c_s^t)\|^2 + \left(\sum_{k, u_{ks = 0}} \frac{ u_{ks}^tn_k}{\sum_{k'} u_{k's}^tn_{k'}} r \Delta\right)\|\nabla F_s(c_s^t)\| \\
    +& \frac{L_F}{2}\left(\frac{\gamma + 1}{\mu_\lambda}\right)^2 \sum_{k, u_{ks \neq 0}} \frac{u_{ks}^tn_k}{u_{ks} \sum_{k'} u_{k's}^tn_{k'}} \|\nabla F_s(c_s^t)\|^2 + \frac{L_F}{2}\sum_{k, u_{ks = 0}} \frac{ u_{ks}^tn_k}{\sum_{k'} u_{k's}^tn_{k'}} r^2 \Delta^2
\end{split}
\end{equation}

Taking expectations over $\{u_{ks}^t\}$, and applying Lemma \ref{lm:exp_ind}, we have
\begin{equation}
\begin{split}
    &\mathbb{E}[F_s(\Bar{c}_s^{t+1})] \leq F_s(c_s^t) - \mathbb{E}\left[\frac{1}{\sum_k u_{ks}^tn_k}\right]\Bigg\{\left(\frac{n_s - \gamma m_s}{\lambda} - \frac{(\gamma + 1)L_F m_s}{\mu_\lambda \lambda}\right)\|\nabla F_s(c_s^t)\|^2\\
    -& \left( p_\epsilon' \Bar{m}_s r \Delta\right)\|\nabla F_s(c_s^t)\| - \frac{L_F}{2}\left(\frac{\gamma + 1}{\mu_\lambda}\right)^2 \hat{m}_s\|\nabla F_s(c_s^t)\|^2 - \frac{L_F}{2}p_\epsilon' \Bar{m}_s r^2 \Delta^2\Bigg\} \\
    \leq& F_s(c_s^t) -  \mathbb{E}\left[\frac{1}{\sum_k u_{ks}^tn_k}\right]\left(\frac{n_s - \gamma m_s}{\lambda} - \frac{(\gamma + 1)L_F m_s}{\mu_\lambda \lambda} - \frac{p_\epsilon' \Bar{m}_s}{2\mu_\lambda} - \frac{L_F(\gamma+1)^2\hat{m}_s}{2\mu_\lambda^2}\right)\|\nabla F_s(c_s^t)\|^2 \\
    +& \mathbb{E}\left[\frac{1}{\sum_k u_{ks}^tn_k}\right]\frac{p_\epsilon'}{2}\left(\mu_\lambda + L_F \right)\Bar{m}_s r^2\Delta^2 
\end{split}
\end{equation}

Define 
\begin{equation}
    \rho_0 \overset{\Delta}{=} \frac{n_s - \gamma m_s}{\lambda} - \frac{(\gamma + 1)L_F m_s}{\mu_\lambda \lambda} - \frac{p_\epsilon' \Bar{m}_s}{2\mu_\lambda} - \frac{L_F(\gamma+1)^2\hat{m}_s}{2\mu_\lambda^2}
\end{equation}
\begin{equation}
    R_0 \overset{\Delta}{=} \frac{1}{2}\left(\mu_\lambda + L_F \right)\Bar{m}_s r^2
\end{equation}

then
\begin{equation} \label{eq:noselection_convergence}
\begin{split}
    &\mathbb{E}[F_s(\Bar{c}_s^{t+1})] \leq F_s(c_s^t) -  \mathbb{E}\left[\frac{1}{\sum_k u_{ks}^tn_k}\right]\rho_0\|\nabla F_s(c_s^t)\|^2 + \mathbb{E}\left[\frac{1}{\sum_k u_{ks}^tn_k}\right] p_\epsilon' R_0 \Delta^2
\end{split}
\end{equation}

Next we incorporate the client selection.

We can write $c_s^{t+1} = \frac{1}{K} \sum_{l=1}^K \hat{c}_{s,l}^{t+1}$, where $\hat{c}_{s,l}^{t+1} = \sum_{k=1}^N \boldsymbol{1}(k \in Sel_{s,l}^t) w_k^{t+1}$, and
\begin{equation}
    \mathbb{P}(k \in Sel_{s,l}^t) = \frac{u_{ks}^tn_k}{\sum_{k'} u_{k's}^tn_{k'}} = v_{sk}^t
\end{equation}

Thus,
\begin{equation}
    \mathbb{E}_{Sel_s^t}[\hat{c}_{s,l}^{t+1}] = \mathbb{E}_{{v_{sk}^t}}[w_k^{t+1}]
\end{equation}

We then bound the difference between $F_s(c_s^{t+1})$ and $F_s(\Bar{c}_s^{t+1})$
\begin{equation} \label{eq:Qst}
\begin{split}
    &F_s(c_s^{t+1}) \leq F_s(\Bar{c}_s^{t+1}) + \langle \nabla F_s(\Bar{c}_s^{t+1}), c_s^{t+1} - \Bar{c}_s^{t+1}\rangle + \frac{L_F}{2} \|c_s^{t+1} - \Bar{c}_s^{t+1}\|^2 \\
    \leq& F_s(\Bar{c}_s^{t+1}) + \left( \| \nabla F_s(\Bar{c}_s^{t+1}) \|  + \frac{L_F}{2} \|c_s^{t+1} - \Bar{c}_s^{t+1}\| \right) \| c_s^{t+1} - \Bar{c}_s^{t+1} \| \\
    \leq& F_s(\Bar{c}_s^{t+1}) + \left( \| \nabla F_s(\Bar{c}_s^{t+1}) - \nabla F_s(c_s^t)\| + \|\nabla F_s(c_s^t) \|  + \frac{L_F}{2}\left( \|c_s^{t+1} - c_s^t\| + \|\Bar{c}_s^{t+1} - c_s^t \|\right) \right) \| c_s^{t+1} - \Bar{c}_s^{t+1} \| \\
    \leq& F_s(\Bar{c}_s^{t+1}) + \left( \|\nabla F_s(c_s^t) \| + L_F\|\Bar{c}_s^{t+1} - c_s^t \|  + \frac{L_F}{2}\left( \|c_s^{t+1} - c_s^t\| + \|\Bar{c}_s^{t+1} - c_s^t \|\right) \right) \| c_s^{t+1} - \Bar{c}_s^{t+1} \| \\
    =& F_s(\Bar{c}_s^{t+1}) + \underbrace{\left( \|\nabla F_s(c_s^t) \|  + \frac{L_F}{2}\|c_s^{t+1} - c_s^t\| + \frac{3L_F}{2} \|\Bar{c}_s^{t+1} - c_s^t \| \right) \| c_s^{t+1} - \Bar{c}_s^{t+1} \|}_{Q_s^t}.
\end{split}
\end{equation}

Thus, we only need to bound $\mathbb{E}[Q_s^t]$
\begin{equation}
\begin{split}
    &\mathbb{E}_{Sel_s^t}[Q_s^t] = \left(\|\nabla F_s(c_s^t) \| + \frac{3L_F}{2} \|\Bar{c}_s^{t+1} - c_s^t \| \right) \mathbb{E}_{Sel_s^t}[\| c_s^{t+1} - \Bar{c}_s^{t+1} \|] + \frac{L_F}{2}\mathbb{E}_{Sel_s^t}[\|c_s^{t+1} - c_s^t\| \cdot \| c_s^{t+1} - \Bar{c}_s^{t+1} \|] \\
    \leq& \left(\|\nabla F_s(c_s^t) \| + 2L_F \|\Bar{c}_s^{t+1} - c_s^t \| \right) \mathbb{E}_{Sel_s^t}[\| c_s^{t+1} - \Bar{c}_s^{t+1} \|] + \frac{L_F}{2}\mathbb{E}_{Sel_s^t}[\| c_s^{t+1} - \Bar{c}_s^{t+1} \|^2] \\
    \leq& \left(\|\nabla F_s(c_s^t) \| + 2L_F \|\Bar{c}_s^{t+1} - c_s^t \| \right) \sqrt{\mathbb{E}_{Sel_s^t}[\| c_s^{t+1} - \Bar{c}_s^{t+1} \|^2]} + \frac{L_F}{2}\mathbb{E}_{Sel_s^t}[\| c_s^{t+1} - \Bar{c}_s^{t+1} \|^2]
\end{split}
\end{equation}

Note that $\mathbb{E}_{v_s^t}[w_k^{t+1}] = \Bar{c}_s^{t+1}$, we have
\begin{equation}
\begin{split}
    &\mathbb{E}_{Sel_s^t}[\| c_s^{t+1} - \Bar{c}_s^{t+1} \|^2] = \frac{1}{K^2}\mathbb{E}_{Sel_{s,l}^t}\left[\left\|\sum_{l=1}^K\left(\hat{c}_{s,l}^{t+1} - \Bar{c}_s^{t+1}\right) \right\|^2\right] \leq \frac{2}{K^2}\sum_{l=1}^K\mathbb{E}_{Sel_{s,l}^t}\left[\left\|\hat{c}_{s,l}^{t+1} - \Bar{c}_s^{t+1}\right\|^2\right] \\
    =& \frac{2}{K}\mathbb{E}_{{v_{sk}^t}}\left[\left\|w_k^{t+1} - \Bar{c}_s^{t+1}\right\|^2\right] = \frac{2}{K} \mathbb{E}_{{v_{sk}^t}} \left[\left\|w_k^{t+1} - c_s^t \|^2 - 2\langle w_k^{t+1} - c_s^t, \Bar{c}_s^{t+1} - c_s^t \rangle + \| \Bar{c}_s^{t+1} - c_s^t\right\|^2\right] \\
    =& \frac{2}{K} \mathbb{E}_{{v_{sk}^t}} \left[\left\|w_k^{t+1} - c_s^t \|^2 - \| \Bar{c}_s^{t+1} - c_s^t\right\|^2\right] \leq \frac{2}{K} \mathbb{E}_{{v_{sk}^t}} \left[\left\|w_k^{t+1} - c_s^t \right\|^2 \right]
\end{split}
\end{equation}

Combining with (\ref{eq:|cs_-cs|}), we can obtain
\begin{equation}
\begin{split}
    &\mathbb{E}_{Sel_s^t}[Q_s^t] \leq \sqrt{\frac{2}{K}} \sqrt{\mathbb{E}_{{v_{sk}^t}} \left[\left\|w_k^{t+1} - c_s^t \right\|^2 \right]} \|\nabla F_s(c_s^t) \| + \left(2L_F\sqrt{\frac{2}{K}} + \frac{L_F}{K}\right) \mathbb{E}_{{v_{sk}^t}} \left[\left\|w_k^{t+1} - c_s^t \right\|^2 \right] \\
    \leq& \frac{1}{2}\sqrt{\frac{2}{K}}\left( \mu_\lambda \mathbb{E}_{{v_{sk}^t}} \left[\left\|w_k^{t+1} - c_s^t \right\|^2 \right] + \frac{1}{\mu_\lambda} \|\nabla F_s(c_s^t) \|^2 \right) + \left(2L_F\sqrt{\frac{2}{K}} + \frac{L_F}{K}\right) \mathbb{E}_{{v_{sk}^t}} \left[\left\|w_k^{t+1} - c_s^t \right\|^2 \right]
\end{split}
\end{equation}

where
\begin{equation}
\begin{split}
    &\mathbb{E}\left[\mathbb{E}_{v_{sk}^t}[\|w_k^{t+1} - c_s^t\|^2]\right] \leq \mathbb{E}\left[\frac{1}{\sum_{k} u_{ks}^tn_{k}}\right]\left\{\left(\frac{\gamma + 1}{\mu_\lambda}\right)^2 \hat{m}_s\|\nabla F_s(c_s^t)\|^2 + p_\epsilon' \Bar{m}_s r^2\Delta^2\right\}
\end{split}
\end{equation}

hence,
\begin{equation} \label{eq:E[Qst]}
\begin{split}
    \mathbb{E}[Q_s^t] \leq& \frac{1}{\mu_\lambda}\sqrt{\frac{1}{2K}}\left((\gamma + 1)^2\hat{m}_s\mathbb{E}\left[\frac{1}{\sum_k u_{ks}^t n_k}\right] + 1\right) \|\nabla F_s(c_s^t) \|^2  \\
    +& \underbrace{\left(2L_F\sqrt{\frac{2}{K}} + \frac{L_F}{K} + \frac{\mu_\lambda}{2} \sqrt{\frac{2}{K}}\right)}_{\leq \frac{4L_F+\mu_\lambda}{\sqrt{K}}} \mathbb{E}\left[\frac{1}{\sum_k u_{ks}^t n_k}\right] p_\epsilon' \Bar{m}_s r^2\Delta^2 \\
    +& \underbrace{\left(2L_F\sqrt{\frac{2}{K}} + \frac{L_F}{K} \right)}_{ \leq \frac{4L_F}{\sqrt{K}}} \mathbb{E}\left[\frac{1}{\sum_k u_{ks}^t n_k}\right] \left(\frac{\gamma + 1}{\mu_\lambda}\right)^2 \hat{m}_s \|\nabla F_s(c_s^t)\|^2
\end{split}
\end{equation}

Combining (\ref{eq:noselection_convergence}), (\ref{eq:Qst}), and (\ref{eq:E[Qst]}) we have
\begin{equation}
\begin{split}
    &\mathbb{E}\left[F_s(c_s^{t+1})\right] \leq \mathbb{E}\left[F_s(\Bar{c}_s^{t+1})\right] + \mathbb{E}[Q_s^t] \\
    \leq& F_s(c_s^t) -  \mathbb{E}\left[\frac{1}{\sum_k u_{ks}^tn_k}\right]\left(\rho_0 - \frac{4L_F(\gamma+1)^2\hat{m}_s}{\mu_\lambda^2\sqrt{K}} - \frac{(\gamma+1)^2\hat{m}_s}{\mu_\lambda\sqrt{2K}} \right)\|\nabla F_s(c_s^t)\|^2 + \frac{1}{\mu_\lambda\sqrt{2K}}\|\nabla F_s(c_s^t)\|^2 \\
    +& \mathbb{E}\left[\frac{1}{\sum_k (1 - \sum_{s' \neq s}u_{ks}^t)n_k}\right]p_\epsilon' \left(R_0 + \frac{(4L_F + \mu_\lambda)\Bar{m}_sr^2}{\sqrt{K}}\right) \Delta^2 \\
\end{split}
\end{equation}

Let $\lambda$ be chosen such that $\rho_0 - \frac{4L_F(\gamma+1)^2\hat{m}_s}{\mu_\lambda^2\sqrt{K}} - \frac{(\gamma+1)^2\hat{m}_s}{\mu_\lambda\sqrt{2K}} > 0$, thus

\begin{equation}
    \mathbb{E}\left[F_s(c_s^{t+1})\right] \leq F_s(c_s^t) -  \frac{\left(\rho_0 - \frac{4L_F(\gamma+1)^2\hat{m}_s}{\mu_\lambda^2\sqrt{K}} - \frac{(\gamma+1)^2\hat{m}_s + (1-p_\epsilon)n_s + p_\epsilon' n}{\mu_\lambda\sqrt{2K}} \right)}{(1-p_\epsilon)n_s + p_\epsilon' n}\|\nabla F_s(c_s^t)\|^2 + \frac{p_\epsilon' \left(R_0 + \frac{(4L_F + \mu_\lambda)\Bar{m}_sr^2}{\sqrt{K}}\right) \Delta^2}{(1-p_\epsilon)n_s - p_\epsilon'(S-2) n}
\end{equation}
\end{proof}

\subsection{Proof of Theorem \ref{tm:joint_convergence}}
Note that under Assumption \ref{as:convex_smooth}, the sum of proximal objectives $h(w_1 \cdots w_N,c_1 \cdots c_S) = \sum_{k=1}^N h_k(w_k;c,\Tilde{u}_k)$ is jointly convex on $(w_1\cdots w_N, c_1\cdots c_S)$ for fixed $\Tilde{u}_k$, and the training process of \textbf{FedSoft} can be regarded as a cyclic block coordinate descent algorithm that sequentially updates $w_1\cdots w_N, c_1\cdots c_S$ while other blocks are fixed. This type of algorithm is know to converge at least linearly to a stationary point \cite{convergence_bcd}. 

To see why the averaging of centers correspond to the minimization of them, simply set the gradients to zero
\begin{equation}
    \nabla_{c_s} h = \sum_{k=1}^N \Tilde{u}_{ks} \left(c_s - w_k\right) = 0
\end{equation}

This implies the optimal $c_s^*$ equals
\begin{equation}
    c_s^* = \sum_k \frac{\Tilde{u}_{ks} w_k}{\sum_k \Tilde{u}_{ks}}
\end{equation}

which is exactly the updating rule of \textbf{FedSoft}.

\section{The impact of $\tau$ on the convergence} \label{sec:appendix_tau}
Note that $\tau$ only affects the accuracy of the importance weight estimations $u_{ks}^t$, which determines the estimation error $p_\epsilon$.

To incorporate $\tau$ into the analysis, we first generalize Assumption \ref{as:init} as follows \cite{hard_cluster_ucb}
\begin{equation}
    \|c_s^t - c_s^* \| \leq \left(0.5 - \alpha_t \right)\sqrt{\mu_F / L_F} \delta, \forall s
\end{equation}
where $0 < \alpha_t \leq 0.5$ for all $t$.

If the algorithm works correctly, the distance between $c_s^t$ and $c_s^*$ should decrease over time, thus $\alpha_t$ will gradually increase from $\alpha_0$ to $0.5$. Then we can change the definition of $p_\epsilon$ in Lemma \ref{lm:p_epsilon}:
\begin{equation}
    \mathbb{P}(\mathcal{E}_t^{j,j'}) \leq p_\epsilon^{t,\tau} \overset{\Delta}{=} \frac{c_\epsilon}{\alpha_{\tau\lfloor t/\tau\rfloor}^2\delta^4} 
\end{equation}
Here we replace $\alpha_0$ with $\alpha_{\tau\lfloor t/\tau\rfloor}$, which takes the same value within each estimation interval. Since we expect $\alpha_t$ to be increasing, we have $\alpha_{\tau\lfloor t/\tau\rfloor} \leq \alpha_t$. Thus, the estimation error $p_\epsilon^{t,\tau}$ increases when we choose a larger interval $\tau$. Plugging this new definition of $p_\epsilon^{t,\tau}$ to Corollary \ref{coro:convergence} we have
\begin{equation}
\begin{split}
    &\sum_{t=1}^T \frac{\rho^{t,\tau} \mathbb{E}\|\nabla F_s(c_s^t)\|^2}{(1-p_\epsilon^{t,\tau})n_s + (p_\epsilon^{t,\tau})' n} \leq \frac{B_s}{T} + O(p_\epsilon^{0,\tau} \Delta^2)
\end{split}
\end{equation}
where $\rho^{t,\tau}$ is defined by replacing all $p_\epsilon$ with $p_\epsilon^{t,\tau}$. This gives us the same asymptotic convergence rate with time as in Corollary \ref{coro:convergence}, except for small differences in constant terms.

\newpage
\section{Experiment Details} \label{sec:appendix_exp}
\subsection{Experiment parameters}
\begin{itemize}
    \item \textbf{Synthetic data}. We use a conventional linear regression model without the intercept term. All clients use Adam as the local solver. The number of local epochs equals 10, batch size equals 10, and the initial learning rate equals 5e-3. The same solver is used for both \textbf{FedSoft} and the baselines. The regularization weight $\lambda = 1.0$ for \textbf{FedSoft}. The training lasts for 50 global epochs.
    \item \textbf{Letters data}. We use a CNN model comprising two convolutional layers with kernel size equal to 5 and padding equal to 2, each followed by the max-pooling with kernel size equal to 2, then connected to a fully-connected layer with 512 hidden neurons followed by ReLU. All clients use Adam as the local solver, with the number of local epochs equal to 5, batch size equal to 5, and initial learning rate equal to 1e-5. The same solver is used for both \textbf{FedSoft} and the baselines. The regularization weight $\lambda = 0.1$ for \textbf{FedSoft}. The training lasts for 200 global epochs.
\end{itemize}

\subsection{Impact of regularization weight $\lambda$}
In previous experiments on the letters dataset, we choose $\lambda=0.1$, which is selected through grid search. We show in Table \ref{tab:letters_varying_lambda} the accuracy of cluster and client models for different choices of $\lambda$ on the letters dataset, while all other parameters are kept unchanged. As we can see, when $\lambda=0$, no global knowledge is passed to clients, thus the local training is done separately without any cooperation, resulting in poorly trained models. On the other hand, when $\lambda$ is increased to 1, the local updating is dominated by fitting the local model to the average of global models, and the local knowledge is less emphasized, which also reduces the algorithm performance.

\begin{table}[h!]
\centering
\caption{Accuracy of cluster and client models for different choices of $\lambda$ for the linear partition of lower and uppercase letters. All the other parameters are kept unchanged.} \label{tab:letters_varying_lambda}

\addtolength{\tabcolsep}{-.5pt}
\begin{tabular}{cccccccccc}
\Xhline{2\arrayrulewidth}
\multirow{2}{*}{} & \multicolumn{3}{c}{$\lambda=0$} & \multicolumn{3}{c}{$\lambda=0.1$} & \multicolumn{3}{c}{$\lambda=1$} \\ \cline{2-10}
                  & $c_0$       & $c_1$       & $\Bar{w}$       & $c_0$       & $c_1$        & $\Bar{w}$       & $c_0$        & $c_1$      & $\Bar{w}$       \\ \hline
Lower           & 48.4       & \underline{50.1}       & -       & \underline{71.8}       &  71.7       & -      & 55.1     & \underline{55.2}  & -       \\
Upper          & \underline{47.7}       & 47.5       & -       & 73.9       &  \underline{74.1}       &  -     & \underline{58.7}       & 58.6   & -       \\
Local          & -       & -       & 72.7       & -       &  -       &  86.5     & -       & -   & 63.6       \\
\Xhline{2\arrayrulewidth}
\end{tabular}

\end{table}

\subsection{Impact of divergence among distributions $\Delta$}
Finally, we show how the divergence among different distributions $\Delta$ affects the performance of \textbf{FedSoft}. For this experiment we use the synthetic dataset, and we control the divergence by choosing different values of $\sigma_0$ (i.e. $\Delta$ increases with $\sigma_0$). As we can see, the MSE significantly increases as the divergence between distributions gets larger, which validates Theorem \ref{tm:convergence}.

\begin{table}[h!]
\centering
\caption{MSE of cluster and client models for different choices of $\sigma_0$ for the mixture of two synthetic distributions under the random partition. All the other parameters are kept unchanged.} \label{tab:letters_varying_lambda}

\addtolength{\tabcolsep}{-.5pt}
\begin{tabular}{ccccccccccccc}
\Xhline{2\arrayrulewidth}
\multirow{2}{*}{} & \multicolumn{3}{c}{$\sigma_0=1$} & \multicolumn{3}{c}{$\sigma_0=10$} & \multicolumn{3}{c}{$\sigma_0=50$} & \multicolumn{3}{c}{$\sigma_0=100$} \\ \cline{2-13}
                  & $c_0$       & $c_1$       & $\Bar{w}$       & $c_0$       & $c_1$        & $\Bar{w}$       & $c_0$        & $c_1$      & $\Bar{w}$     &  $c_0$        & $c_1$      & $\Bar{w}$      \\ \hline
$\theta_0$           & 5.9       & \underline{4.2}       & -       & \underline{42.2}       &  60.6       & -      & 225.3     & \underline{89.9}  & -   & 782.4     & \underline{454.4}  & -       \\
$\theta_1$          & \underline{3.2}       & 4.6       & -       & 42.7       &  \underline{27.0}       &  -     & \underline{117.6}       & 89.9   & -   & \underline{432.8}       & 812.0   & -       \\
Local          & -       & -       & 0.6       & -       &  -       &  4.96     & -       & -   & 18.8   & -       & -   & 78.3       \\
\Xhline{2\arrayrulewidth}
\end{tabular}
\end{table}

\subsection{CIFAR-10 Results}
In this section we evaluate the performance of \textbf{FedSoft} for the CIFAR-10 dataset. We consider two data distributions: the original CIFAR-10 images and their rotation by 90\degree counterclockwise. All images are preprocesseed with standard data augmentation tools \cite{hard_cluster_ucb}. We use a CNN model with six convolutional layers, whose channel sizes are sequentially 32, 64, 128, 128, 256, 256. Each convolutional layer follows the ReLU activation and every two convolutional layer follows a max pool layer. The fully connected layer has dimension 1024$\times$512.

20 clients are used for this experiment. We choose client selection size $K=15$, importance estimation interval $\tau=2$, regularization weight $\lambda=0.01$. The local solvers are Adam with initial learning rate equals 5e-4, number of local epochs equals 10, batch size equals 64. The training lasts for 200 global epochs.
\newpage

\begin{table}[t!]
\centering
\caption{Accuracy of cluster models for the mixture of two CIFAR-10 distributions. Each row represents the distribution of a test dataset. The center with the highest accuracy is underlined for each test distribution.} \label{tab:cifar_centers_2partition}

\addtolength{\tabcolsep}{-1.2pt}
\begin{tabular}{ccccccccc}
\Xhline{2\arrayrulewidth}
\multirow{2}{*}{} & \multicolumn{2}{c}{10:90} & \multicolumn{2}{c}{30:70} & \multicolumn{2}{c}{Linear} & \multicolumn{2}{c}{Random} \\ \cline{2-9}
                  & $c_0$       & $c_1$       & $c_0$       & $c_1$       & $c_0$        & $c_1$       & $c_0$        & $c_1$       \\ \hline
$0\degree$           & 74.8       & \underline{77.6}       & \underline{78.9}       & \underline{78.9}       & 77.2        & \underline{77.4}      & \underline{76.1}        & 76.0       \\
$90\degree$          & \underline{77.4}       & 75.3       & \underline{79.0}       & 78.8       & \underline{78.5}        & 78.0      & 75.3        & \underline{75.8}       \\ \Xhline{2\arrayrulewidth}
\end{tabular}
\end{table}

\begin{table}[t!]
\centering
\caption{Comparison between FedSoft and baselines on the CIFAR-10 data. $c_{0}^*$/$c_{90}^*$ represents the accuracy of the center that performs best on the 0\degree/90\degree distribution, and the number in the parenthesis indicates the index of that center. $\Bar{w}$ is the accuracy of local models averaged over all clients.} \label{tab:cifar_two_mixture_comparison}
\addtolength{\tabcolsep}{-2pt}
\begin{tabular}{ccccccc}
\Xhline{2\arrayrulewidth}
\multirow{2}{*}{} & \multicolumn{3}{c}{30:70} & \multicolumn{3}{c}{Linear} \\ \cline{2-7}
                  & $c_{0}^*$       & $c_{90}^*$       & $\Bar{w}$       & $c_{0}^*$       & $c_{90}^*$        & $\Bar{w}$       \\ \hline
FedSoft           & 78.9(1)       & 79.0(0)       &98.8       & 77.4(1)       & 78.5(0)        & 98.6      \\
IFCA          & 75.6(0)       & 76.1(0)       & 99.0       & 75.1(0)       & 75.3(0)        & 98.8       \\
FedEM          & 76.0(1)       & 74.6(1)      & 96.6  & 76.3(0)       & 75.8(0)        & 82.0       \\\Xhline{2\arrayrulewidth}
\end{tabular}
\end{table}

Table \ref{tab:cifar_centers_2partition} shows the accuracy of cluster models for all the four data partition patterns presented in Section \ref{sec:exp}. As we can see, \textbf{FedSoft} yields high quality cluster models with a clear separation of different distributions. Table \ref{tab:cifar_two_mixture_comparison} compares \textbf{FedSoft} with \textbf{IFCA} and \textbf{FedEM}. \textbf{FedSoft} has the best performance for cluster models, and produces fairly accurate local models.

%% file: 0-main.bbl
\begin{thebibliography}{23}
\providecommand{\natexlab}[1]{#1}

\bibitem[{Briggs et~al.(2020)}]{flhc}
Briggs, C.; et~al. 2020.
\newblock Federated learning with hierarchical clustering of local updates to
  improve training on non-IID data.
\newblock In \emph{2020 International Joint Conference on Neural Networks
  (IJCNN)}, 1--9. IEEE.

\bibitem[{Brisimi et~al.(2018)Brisimi, Chen, Mela, Olshevsky, Paschalidis, and
  Shi}]{medical_fl}
Brisimi, T.~S.; Chen, R.; Mela, T.; Olshevsky, A.; Paschalidis, I.~C.; and Shi,
  W. 2018.
\newblock Federated learning of predictive models from federated electronic
  health records.
\newblock \emph{International journal of medical informatics}, 112: 59--67.

\bibitem[{Dinh, Tran, and Nguyen(2020)}]{pfedme}
Dinh, C.~T.; Tran, N.~H.; and Nguyen, T.~D. 2020.
\newblock Personalized federated learning with moreau envelopes.
\newblock \emph{arXiv preprint arXiv:2006.08848}.

\bibitem[{Duan et~al.(2020)Duan, Liu, Ji, Liu, Liang, Chen, and Tan}]{fedgroup}
Duan, M.; Liu, D.; Ji, X.; Liu, R.; Liang, L.; Chen, X.; and Tan, Y. 2020.
\newblock FedGroup: Ternary Cosine Similarity-based Clustered Federated
  Learning Framework toward High Accuracy in Heterogeneous Data.
\newblock \emph{arXiv preprint arXiv:2010.06870}.

\bibitem[{Ghosh et~al.(2020)Ghosh, Chung, Yin, and
  Ramchandran}]{hard_cluster_ucb}
Ghosh, A.; Chung, J.; Yin, D.; and Ramchandran, K. 2020.
\newblock An efficient framework for clustered federated learning.
\newblock \emph{arXiv preprint arXiv:2006.04088}.

\bibitem[{Glorot and Bengio(2010)}]{xavier}
Glorot, X.; and Bengio, Y. 2010.
\newblock Understanding the difficulty of training deep feedforward neural
  networks.
\newblock In \emph{Proceedings of the thirteenth international conference on
  artificial intelligence and statistics}, 249--256. JMLR Workshop and
  Conference Proceedings.

\bibitem[{Guha, Talwalkar, and Smith(2019)}]{oneshot}
Guha, N.; Talwalkar, A.; and Smith, V. 2019.
\newblock One-shot federated learning.
\newblock \emph{arXiv preprint arXiv:1902.11175}.

\bibitem[{Hard et~al.(2018)Hard, Rao, Mathews, Ramaswamy, Beaufays, Augenstein,
  Eichner, Kiddon, and Ramage}]{next_word}
Hard, A.; Rao, K.; Mathews, R.; Ramaswamy, S.; Beaufays, F.; Augenstein, S.;
  Eichner, H.; Kiddon, C.; and Ramage, D. 2018.
\newblock Federated learning for mobile keyboard prediction.
\newblock \emph{arXiv preprint arXiv:1811.03604}.

\bibitem[{Huang et~al.(2019)Huang, Shea, Qian, Masurkar, Deng, and
  Liu}]{patient_cluster}
Huang, L.; Shea, A.~L.; Qian, H.; Masurkar, A.; Deng, H.; and Liu, D. 2019.
\newblock Patient clustering improves efficiency of federated machine learning
  to predict mortality and hospital stay time using distributed electronic
  medical records.
\newblock \emph{Journal of biomedical informatics}, 99: 103291.

\bibitem[{Huang et~al.(2021)Huang, Chu, Zhou, Wang, Liu, Pei, and
  Zhang}]{fedamp}
Huang, Y.; Chu, L.; Zhou, Z.; Wang, L.; Liu, J.; Pei, J.; and Zhang, Y. 2021.
\newblock Personalized cross-silo federated learning on non-iid data.
\newblock In \emph{Proceedings of the AAAI Conference on Artificial
  Intelligence}, volume~35, 7865--7873.

\bibitem[{Kairouz et~al.(2019)Kairouz, McMahan, Avent, Bellet, Bennis, Bhagoji,
  Bonawitz, Charles, Cormode, Cummings et~al.}]{advances}
Kairouz, P.; McMahan, H.~B.; Avent, B.; Bellet, A.; Bennis, M.; Bhagoji, A.~N.;
  Bonawitz, K.; Charles, Z.; Cormode, G.; Cummings, R.; et~al. 2019.
\newblock Advances and open problems in federated learning.
\newblock \emph{arXiv preprint arXiv:1912.04977}.

\bibitem[{Li et~al.(2018)Li, Sahu, Zaheer, Sanjabi, Talwalkar, and
  Smith}]{fedprox}
Li, T.; Sahu, A.~K.; Zaheer, M.; Sanjabi, M.; Talwalkar, A.; and Smith, V.
  2018.
\newblock Federated optimization in heterogeneous networks.
\newblock \emph{arXiv preprint arXiv:1812.06127}.

\bibitem[{Lopez-Paz and Ranzato(2017)}]{mnist_rotation}
Lopez-Paz, D.; and Ranzato, M. 2017.
\newblock Gradient episodic memory for continual learning.
\newblock \emph{Advances in neural information processing systems}, 30:
  6467--6476.

\bibitem[{Luo and Tseng(1992)}]{convergence_bcd}
Luo, Z.-Q.; and Tseng, P. 1992.
\newblock On the convergence of the coordinate descent method for convex
  differentiable minimization.
\newblock \emph{Journal of Optimization Theory and Applications}, 72(1): 7--35.

\bibitem[{Mansour et~al.(2020)Mansour, Mohri, Ro, and
  Suresh}]{three_approaches}
Mansour, Y.; Mohri, M.; Ro, J.; and Suresh, A.~T. 2020.
\newblock Three approaches for personalization with applications to federated
  learning.
\newblock \emph{arXiv preprint arXiv:2002.10619}.

\bibitem[{Marfoq et~al.(2021)Marfoq, Neglia, Bellet, Kameni, and
  Vidal}]{fmtl_mixture}
Marfoq, O.; Neglia, G.; Bellet, A.; Kameni, L.; and Vidal, R. 2021.
\newblock Federated Multi-Task Learning under a Mixture of Distributions.
\newblock \emph{International Workshop on Federated Learning for User Privacy
  and Data Confidentiality in Conjunction with ICML 2021 (FL-ICML'21)}.

\bibitem[{McMahan et~al.(2017)McMahan, Moore, Ramage, Hampson, and
  y~Arcas}]{fedavg}
McMahan, B.; Moore, E.; Ramage, D.; Hampson, S.; and y~Arcas, B.~A. 2017.
\newblock Communication-efficient learning of deep networks from decentralized
  data.
\newblock In \emph{Artificial Intelligence and Statistics}, 1273--1282. PMLR.

\bibitem[{Qayyum et~al.(2021)Qayyum, Ahmad, Ahsan, Al-Fuqaha, and
  Qadir}]{covid}
Qayyum, A.; Ahmad, K.; Ahsan, M.~A.; Al-Fuqaha, A.; and Qadir, J. 2021.
\newblock Collaborative federated learning for healthcare: Multi-modal covid-19
  diagnosis at the edge.
\newblock \emph{arXiv preprint arXiv:2101.07511}.

\bibitem[{Reynolds(2009)}]{gaussian_mixture}
Reynolds, D.~A. 2009.
\newblock Gaussian mixture models.
\newblock \emph{Encyclopedia of biometrics}, 741: 659--663.

\bibitem[{Sattler, M{\"u}ller, and Samek(2020)}]{cfl}
Sattler, F.; M{\"u}ller, K.-R.; and Samek, W. 2020.
\newblock Clustered federated learning: Model-agnostic distributed multitask
  optimization under privacy constraints.
\newblock \emph{IEEE transactions on neural networks and learning systems}.

\bibitem[{Sim, Zadrazil, and Beaufays(2019)}]{fine_tune}
Sim, K.~C.; Zadrazil, P.; and Beaufays, F. 2019.
\newblock An investigation into on-device personalization of end-to-end
  automatic speech recognition models.
\newblock \emph{arXiv preprint arXiv:1909.06678}.

\bibitem[{Smith et~al.(2017)Smith, Chiang, Sanjabi, and Talwalkar}]{fmtl}
Smith, V.; Chiang, C.-K.; Sanjabi, M.; and Talwalkar, A. 2017.
\newblock Federated multi-task learning.
\newblock \emph{arXiv preprint arXiv:1705.10467}.

\bibitem[{Xie et~al.(2021)Xie, Long, Shen, Zhou, Wang, Jiang, and
  Zhang}]{multi-center}
Xie, M.; Long, G.; Shen, T.; Zhou, T.; Wang, X.; Jiang, J.; and Zhang, C. 2021.
\newblock Multi-center federated learning.
\newblock \emph{arXiv preprint arXiv:2108.08647}.

\end{thebibliography}
